\documentclass[10pt,twocolumn,twoside]{IEEEtran}
\usepackage[utf8]{inputenc}
\usepackage{amsmath,amsthm,amsfonts,amssymb,bm,accents,dsfont}
\usepackage{graphicx,cite,enumerate,afterpage}
\usepackage{url}
\graphicspath{{Images/}}

\usepackage{blindtext}

\newcommand{\includesvg}[2][scale=1]{\includegraphics[#1]{#2.pdf}}

\usepackage[english]{babel}

\usepackage{algpseudocode,algorithm}
\algrenewcommand\algorithmicdo{}
\algrenewtext{EndFor}{\algorithmicend}
\algrenewtext{EndProcedure}{\algorithmicend}

\usepackage{color}

\newtheorem{theorem}{Theorem}
\newtheorem{proposition}{Proposition}

\newtheorem{lemma}{Lemma}
\theoremstyle{definition}

\newtheorem{example}{Example}
\newtheoremstyle{assume}
  {3pt}% measure of space to leave above the theorem. e.g.: 3pt
  {3pt}% measure of space to leave below the theorem. e.g.: 3pt
  {}% name of font to use in the body of the theorem
  {}% measure of space to indent
  {\bf}% name of head font
  {}% punctuation between head and body
  { }% space after theorem head; " " = normal interword space
  {\thmname{#1}.\thmnumber{#2}\thmnote{ \textnormal{(\textit{#3})}}}% Manually specify head
\theoremstyle{assume}

\DeclareMathOperator{\E}{\mathbb{E}}

\DeclareMathOperator{\sign}{sign}
\let\Re\relax
\DeclareMathOperator{\Re}{\mathbb{R}e}
\let\Im\relax
\DeclareMathOperator{\Im}{\mathbb{I}m}

\DeclareMathOperator*{\argmin}{argmin}
\DeclareMathOperator*{\argmax}{argmax}
\DeclareMathOperator*{\minimize}{minimize}
\DeclareMathOperator*{\maximize}{maximize}

\DeclareMathOperator{\subjectto}{subject\ to}

\DeclareMathOperator{\supp}{supp}

\DeclareMathOperator{\indicator}{\mathbb{I}}

\newcommand{\vect}[2]{\ensuremath{[\begin{array}{#1} #2 \end{array}]}}
\newcommand{\norm}[1]{\ensuremath{\left\| #1 \right\|}}
\newcommand{\abs}[1]{\ensuremath{{\left\vert #1 \right\vert}}}

\newcommand{\calA}{\ensuremath{\mathcal{A}}}
\newcommand{\calB}{\ensuremath{\mathcal{B}}}
\newcommand{\calC}{\ensuremath{\mathcal{C}}}
\newcommand{\calD}{\ensuremath{\mathcal{D}}}
\newcommand{\calE}{\ensuremath{\mathcal{E}}}

\newcommand{\calL}{\ensuremath{\mathcal{L}}}

\newcommand{\calO}{\ensuremath{\mathcal{O}}}
\newcommand{\calP}{\ensuremath{\mathcal{P}}}

\newcommand{\calS}{\ensuremath{\mathcal{S}}}
\newcommand{\calT}{\ensuremath{\mathcal{T}}}

\newcommand{\calX}{\ensuremath{\mathcal{X}}}

\newcommand{\calZ}{\ensuremath{\mathcal{Z}}}

\newcommand{\bzero}{\ensuremath{\bm{0}}}

\newcommand{\bF}{\ensuremath{\bm{F}}}

\newcommand{\bX}{\ensuremath{\bm{X}}}

\newcommand{\bh}{\ensuremath{\bm{h}}}

\newcommand{\bk}{\ensuremath{\bm{k}}}

\newcommand{\bp}{\ensuremath{\bm{p}}}

\newcommand{\bu}{\ensuremath{\bm{u}}}

\newcommand{\bw}{\ensuremath{\bm{w}}}
\newcommand{\bx}{\ensuremath{\bm{x}}}
\newcommand{\by}{\ensuremath{\bm{y}}}
\newcommand{\bz}{\ensuremath{\bm{z}}}
\newcommand{\bbeta}{\ensuremath{\bm{\beta}}}
\newcommand{\bdelta}{\ensuremath{\bm{\delta}}}
\newcommand{\bmu}{\ensuremath{\bm{\mu}}}
\newcommand{\bnu}{\ensuremath{\bm{\nu}}}

\newcommand{\bphi}{\ensuremath{\bm{\phi}}}

\newcommand{\fkh}{\ensuremath{\mathfrak{h}}}
\newcommand{\fkm}{\ensuremath{\mathfrak{m}}}

\newcommand{\fkp}{\ensuremath{\mathfrak{p}}}

\newcommand{\fkv}{\ensuremath{\mathfrak{v}}}

\newcommand{\setR}{\ensuremath{\mathbb{R}}}
\newcommand{\setC}{\ensuremath{\mathbb{C}}}

\def\st/{\textsuperscript{st}}
\def\nd/{\textsuperscript{nd}}
\def\rd/{\textsuperscript{rd}}
\def\th/{\textsuperscript{th}}

\newcommand{\ones}{\ensuremath{\mathds{1}}}

\newcommand{\MSE}{\textup{MSE}}
\newcommand{\bph}{\ensuremath{\bm{{\hat{p}}}}}
\newcommand{\byh}{\ensuremath{\bm{{\hat{y}}}}}
\newcommand{\bxt}{\ensuremath{\bm{{\tilde{x}}}}}

\def\nnil{\nil}
\newcounter{prob}
\newenvironment{prob}[1][\nil]{%
	\def\tmp{#1}
	\equation
	\ifx\tmp\nnil
		\refstepcounter{prob}
		\tag{P\Roman{prob}}
	\else
		\tag{\tmp}
	\fi
	\aligned%
}{%
	\endaligned\endequation%
}

\title{Functional Nonlinear Sparse Models}

\author{Luiz~F.~O.~Chamon, Yonina C. Eldar, and Alejandro~Ribeiro%
\thanks{L.F.O. Chamon and A. Ribeiro are with the Department of Electrical and Systems Engineering, University of Pennsylvania.
e-mail: \mbox{\texttt{\{luizf,aribeiro\}@seas.upenn.edu}}.}
\thanks{Y.C. Eldar is with the Math and Computer Science Department, Weizmann Institute of Science. e-mail: \mbox{\texttt{yonina.eldar@weizmann.ac.il}}}
\thanks{Part of the results in this paper appeared in~\cite{Chamon18s, Chamon19s}.}%
\thanks{Mr.\ Chamon's and Dr.\ Ribeiro's work is supported by ARL DCIST CRA W911NF-17-2-0181.}
}

\begin{document}
\maketitle
\begin{abstract}

Signal processing is rich in inherently continuous and often nonlinear applications, such as spectral estimation, optical imaging, and super-resolution microscopy, in which sparsity plays a key role in obtaining state-of-the-art results. Coping with the infinite dimensionality and non-convexity of these problems typically involves discretization and convex relaxations, e.g., using atomic norms. Nevertheless, grid mismatch and other coherence issues often lead to discretized versions of sparse signals that are not sparse. Even if they are, recovering sparse solutions using convex relaxations requires assumptions that may be hard to meet in practice. What is more, problems involving nonlinear measurements remain non-convex even after relaxing the sparsity objective. We address these issues by directly tackling the continuous, nonlinear problem cast as a sparse functional optimization program. We prove that when these problems are non-atomic, they have no duality gap and can therefore be solved efficiently using duality and~(stochastic) convex optimization methods. We illustrate the wide range of applications of this approach by formulating and solving problems from nonlinear spectral estimation and robust classification.

\end{abstract}

\begin{IEEEkeywords}
Functional optimization, sparsity, nonlinear models, strong duality, compressive sensing.
\end{IEEEkeywords}

\section{Introduction}
	\label{S:intro}

The analog and often nonlinear nature of the physical world make for two of the main challenges in signal and information processing applications. Indeed, there are many examples of inherently continuous%
\footnote{Throughout this work, we use the term ``continuous'' \emph{only} in contrast to ``discrete'' and not to refer to a smoothness property of signals.}
problems, such as spectral estimation, image recovery, and source localization~\cite{Stoica05s, Pock10g, Ekanadham11r, Barilan14s}, as well as nonlinear ones, e.g., magnetic resonance fingerprinting, spectrum cartography, and manifold data sparse coding~\cite{Ma13m, Bazerque11g, Xie13n}. These challenges are often tackled by imposing structure on the signals. For instance, bandlimited, finite rate of innovation, or union-of-subspaces signals can be processed using an appropriate discrete set of samples~\cite{Unser00s, Vetterli02s, Mishali11x, Eldar15s}. Functions in reproducing kernel Hilbert spaces~(RKHSs) also admit finite descriptions through variational results known as representer theorems~\cite{Kimeldorf70c, Scholkopf01g}. The infinite dimensionality of continuous problems is therefore often overcome by means of sampling theorems. Similarly, nonlinear functions with bounded total variation or lying in an RKHS can be written as a finite linear combination of basis functions. Under certain smoothness assumptions, nonlinearity can be addressed using~``linear-in-the-parameters'' methods.

Due to the limited number of measurements, however, discretization often leads to underdetermined problems. Sparsity priors then play an important role in achieving state-of-the-art results by leveraging the assumption that there exists a signal representation in terms of only a few atoms from an overparametrized dictionary~\cite{Eldar12c, Foucart13m, Eldar15s}. Since fitting these models leads to non-convex~(and possibly NP-hard~\cite{Natarajan95s}) problems, sparsity is typically replaced by a tractable relaxation based on an atomic norm~(e.g., the~$\ell_1$-norm). For linear, incoherent dictionaries, these relaxed problems have be shown to retrieve the desired sparse solution~\cite{Eldar12c, Foucart13m}. Nevertheless, discretized continuous problems rarely meet these conditions. Only in specific instances, such as line spectrum estimation, there exist guarantees for relaxations that forgo discretization~\cite{Mishali11x, Tang13c, Bhaskar13a, Candes14t, Cho15b, Yang15o, Puy17r}.

This discretization/relaxation approach, however, is not always effective. Indeed, discretization can lead to grid mismatch issues and even loss of sparsity: infinite dimensional sparse signals need not be sparse when discretized~\cite{Chi11s, Adcock16g, Adcock17b}. Also, sampling theorems are sensitive to the function class considered and are often asymptotic: results improve as the discretization becomes finer. This leads to high dimensional statistical problems with potentially poor numerical properties~(high condition number). In fact, $\ell_1$-norm-based recovery of spikes on fine grids~(essentially) finds twice the number of actual spikes and the number of support candidate points increases as the number of measurements decreases~\cite{Duval17s1, Duval17s2}. Furthermore, performance guarantees for convex relaxations rely on incoherence assumptions~(e.g., restricted isometry/eigenvalue properties) that may be difficult to meet in practice and are NP-hard to check~\cite{Bandeira13c, Tillmann14c, Natarajan14c}. Finally, these guarantees hold for linear measurements models.

Directly accounting for nonlinearities in sparse models makes a difficult problem harder, since the optimization program remains non-convex even after relaxing the sparsity objective. This is evidenced by the weaker guarantees existing for~$\ell_1$-norm relaxations in nonlinear compressive sensing problems~\cite{Beck13s, Yang16s}. Though ``linear-in-the-parameters'' models, such as splines or kernel methods, may sometimes be used~(e.g., spectrum cartography~\cite{Bazerque11g}), they are not applicable in general. Indeed, the number of kernels needed to represent a generic nonlinear model may be so large that the solution is no longer sparse. What is more, there is no guarantee that these models meet the incoherence assumptions required for the convex relaxation to be effective~\cite{Eldar12c, Foucart13m, Eldar15s}.

In this work, we propose to forgo both discretization and relaxation and directly tackle the continuous problem using \emph{sparse functional programming}. Although sparse functional programs~(SFPs) combine the infinite dimensionality of functional programming with the non-convexity of sparsity and nonlinear atoms, we show that they are tractable under mild conditions~(Theorem~\ref{T:strongDuality}). To do so, this paper develops the theory of sparse functional programming by formulating a general SFP~(Section~\ref{S:sfp}), deriving its Lagrangian dual problem~(Section~\ref{S:dualProblem}), and proving that strong duality holds under mild conditions~(Section~\ref{S:dualityGap}). This shows that SFPs can be solved exactly by means of their dual problems. Moreover, we use this result to obtain a relation between minimizing the support of a function~(``$L_0$-norm'') and its~$L_1$-norm, even though the latter may yield non-sparse solutions~(Section~\ref{S:L1}). We then propose two algorithms to solve SFPs, based on subgradient and stochastic subgradient ascent, by leveraging different numerical integration methods~(Section~\ref{S:dualSolution}). Finally, we illustrate the expressiveness of SFPs by using them to cast different signal processing problems and provide numerical examples to showcase the effectiveness of this approach~(Section~\ref{S:Apps}).

Throughout the paper, we use lowercase boldface letters for vectors~($\bx$), uppercase boldface letters for matrices~($\bX$), calligraphic letters for sets~($\calA$), and fraktur font for measures~($\fkh$). In particular, we denote the Lebesgue measure by~$\fkm$. We use~$\setC$ to denote the set of complex numbers, $\setR$ for real numbers, and~$\setR_+$ for non-negative real numbers. For a complex number~$z = a + jb$, $j = \sqrt{-1}$, we write~$\Re[z] = a$ for its real part and~$\Im[z] = b$ for its imaginary part. We use~$\bz^H$ for the conjugate transpose of the complex vector~$\bz$, $[\bz]_i$ to indicate its~$i$-th component, $\abs{\calA}$ for the cardinality of~$\calA$, and~$\supp(X) = \{\bbeta \in \Omega \mid X(\beta) \neq 0\}$ for the support of~$X: \Omega \to \setC$. For two vectors~$\bx,\by \in \setR^n$, we write~$\bx \succeq \by$ to denote that~$[\bx]_i \geq [\by]_i$ for all~$i = 1,\dots,n$. We define the indicator function~$\indicator: \Omega \to \{0,1\}$ as~$\indicator(\bbeta \in \calE) = 1$, if $\bbeta$ belongs to the event~$\calE$, and zero otherwise.

\section{Sparse functional programs and duality}
	\label{S:sfp}

\subsection{Sparse functional programs}

SFPs are variational problems that seek sparsest functions, i.e., functions with minimum support measure. Explicitly, let~$(\Omega,\calB)$ be a measurable space in which~$\calB$ are the Borel sets of~$\Omega$, a compact set of~$\setR^n$. In a parallel with the discrete case, define the~$L_0$-norm%
\footnote{As in the discrete case, the ``$L_0$-norm'' in~\eqref{E:l0norm} is not a norm. We however omit the quotation marks so as not to burden the text.}
to be the measure of the support of a function, i.e., for a measurable function~$X: \Omega \to \setC$,
\begin{equation}\label{E:l0norm}
	\norm{X}_{L_0} = \fkm\left[ \supp(X) \right] =
		\int_{\Omega} \indicator\left[ X(\bbeta) \neq 0 \right] d\bbeta
		\text{.}
\end{equation}
Note that the integral in~\eqref{E:l0norm} is a multivariate integral over vectors~$\bbeta \in \Omega$. Unless otherwise specified, all integrals are taken with respect to the Lebesgue measure~$\fkm$.

A general SFP is then defined as the optimization problem
\begin{prob}[P-SFP]\label{P:generalSFP}
	\minimize_{X \in \calX,\, \bz \in \setC^p}&
		&&\int_{\Omega} F_0 \left[ X(\bbeta),\bbeta \right] d\bbeta
			+ \lambda \norm{X}_{L_0}
	\\
	\subjectto& &&g_i(\bz) \leq 0
		\text{,} \quad i = 1,\dots,m
	\\
	&&&\bz = \int_{\Omega} \bF \left[ X(\bbeta), \bbeta \right] d\bbeta
	\\
	&&&X(\bbeta) \in \calP \text{ a.e.}
\end{prob}
where~$\lambda > 0$ is a regularization parameter that controls the sparsity of the solution; $g_i: \setC^p \to \setR$ are convex functions; $F_0: \setC \times \Omega \to \setR$ is an optional, not necessarily convex regularization term~(e.g., take~$F_0(x,\bbeta) = \abs{x}^2$ for shrinkage); $\bF: \setC \times \Omega \to \setC^p$ is a vector-valued~(possibly nonlinear) function; $\calP$ is a~(possibly non-convex) set defining an almost everywhere~(a.e.) pointwise constraints on~$X$, i.e., a constraint that holds for all~$\bbeta \in \Omega$ except perhaps over a set of measure zero~(e.g., $\calP = \{x \in \setC \mid \abs{x} \leq \Gamma\}$ for some~$\Gamma > 0$); and~$\calX$ is a \emph{decomposable} function space, i.e., if~$X,X^\prime \in \calX$, then for any~$\calZ \in \calB$ it holds that~$\bar{X} \in \calX$ for
\begin{equation*}
	\bar{X}(\bbeta) =
	\begin{cases}
		X(\bbeta) \text{,} &\bbeta \in \calZ
		\\
		X^\prime(\bbeta) \text{,} &\bbeta \notin \calZ
	\end{cases}
		\text{.}
\end{equation*}
Lebesgue spaces~(e.g., $\calX = L_2$ or~$\calX = L_\infty$) or more generally Orlicz spaces are typical examples of decomposable function spaces. The spaces of constant or continuous functions, for instance, are \emph{not} decomposable~\cite{Rockafellar98v}.

The linear, continuous sparse recovery/denoising problem is a particular case of~\eqref{P:generalSFP}. Here, we seek to represent a signal~$\by \in \setR^p$ as a linear combination of a continuum of atoms~$\phi(\beta)$ indexed by~$\beta \in \Omega \subset \setR$, i.e., as~$\int_\Omega X(\beta) \phi(\beta) d\beta$, using a sparse functional coefficient~$X$. This problem can be posed as
\begin{prob}\label{P:illustration}
	\minimize_{X \in L_2,\, \byh \in \setR^p}&
		&&\norm{X}_{L_2}^2 + \lambda \norm{X}_{L_0}
	\\
	\subjectto& &&\norm{\by - \byh}_2^2 \leq \epsilon
	\\
	&&&\byh = \int_{\Omega} X(\beta) \phi(\beta) d\beta
		\text{,}
\end{prob}
where~$\epsilon > 0$ is a goodness-of-fit parameter. Notice that when discretized, this problem yields the classical~(NP-hard~\cite{Natarajan95s}) dictionary denoising problem with~$\byh = \bphi^T \bx$, where~$\bphi = [\phi(\beta_j)]$ and~$\bx = [X(\beta_j)]$, for a set of~$\beta_j \in \Omega$, $j = 1,\dots,m$.

The expressiveness of SFPs comes from their ability to accommodate nonlinear measurement models~(through~$\bF$) and non-convex objective functions. For instance, \eqref{P:illustration} can also be posed using the nonlinear model
\begin{equation}
	\byh = \int_{\Omega} \rho \left[ X(\beta) \phi(\beta) \right] d\beta
		\text{,}
\end{equation}
where~$\rho$ represents, for instance, a source saturation~(as in, e.g., Section~\ref{S:Apps}). Yet, the abstract formulation in~\eqref{P:generalSFP} certainly obfuscates the applicability of SFPs. Additionally, severe technical challenges, such as infinite dimensionality and non-convexity, appear to hinder their usefulness. We defer the issue of applicability to Section~\ref{S:Apps}, where we illustrate the use of SFPs in the context of nonlinear spectral estimation and nonlinear functional data analysis. Instead, we first focus on whether problems of the form~\eqref{P:generalSFP} can even be solved. Indeed, note that the discrete versions of certain SFPs are known to be NP-hard~\cite{Natarajan95s}. Hence, discretizing the functional problem in this case makes it intractable.

We propose to solve SFPs using duality. It is worth noting that duality is often used to solve semi-infinite convex programs~\cite{Shapiro06d, Tang13c, Bhaskar13a, Candes14t}. In these cases, strong duality holds under mild conditions and solving the dual problem leads to a solution of the original optimization problem of interest. However, SFPs are not convex. To address this issue, we first derive the dual problem of~\eqref{P:generalSFP} in the next section, noting that it is both finite dimensional and convex. Then, we show that we can obtain a solution of~\eqref{P:generalSFP} from a solution of its dual by proving that SFPs have zero duality gap under quite general conditions~(Section~\ref{S:dualityGap}). Finally, we suggest different algorithms to solve the dual problem of~\eqref{P:generalSFP}~(Section~\ref{S:dualSolution}).

\subsection{The Lagrangian dual of sparse functional programs}
	\label{S:dualProblem}

To formulate the dual problem of~\eqref{P:generalSFP}, we first introduce the Lagrange multipliers~$\nu_i \in \setR_+$, corresponding to the inequalities~$g_i(\bz) \leq 0$, and~$\bmu_R,\bmu_I \in \setR^p$, corresponding to the real and imaginary parts respectively of the complex-valued equality~$\bz = \int_{\Omega} \bF \left[ X(\bbeta), \bbeta \right] d\bbeta$. To simplify the derivations, we collect the former into the vector~$\bnu \in \setR_+^m$ and combine the latter two multipliers into a single complex-valued dual variable by noticing that for any vector~$\bx \in \setC^m$, it holds that~$\bmu_R^T \Re[ \bx ] + \bmu_I^T \Im[ \bx ] = \Re \left[ \bmu^H \bx \right]$, where~$\bmu = \bmu_R + j\bmu_I$.

The Lagrangian dual of~\eqref{P:generalSFP} is then defined as
\begin{equation}\label{E:lagrangian}
\begin{aligned}
	\calL(X, \bz, \bmu, \bnu) &=
	\int_{\Omega} F_0 \left[ X(\bbeta),\bbeta \right] d\bbeta
		+ \lambda \norm{X}_{L_0}
	\\
	{}&+ \Re \left[
		\bmu^H \left( \int_{\Omega} \bF \left[ X(\bbeta), \bbeta \right] d\bbeta
		- \bz \right)
	\right]
	\\
	{}&+ \sum_{i = 1}^m \nu_i g_i(\bz)
\end{aligned}
\end{equation}
and its dual function is given by
\begin{equation}\label{E:preDualFunction1}
	d(\bmu, \bnu) = \min_{\substack{X \in \calX,\, \bz \in \setC^p, \\ X(\bbeta) \in \calP}}
		\calL(X, \bz, \bmu, \bnu)
		\text{.}
\end{equation}
The fact that the pointwise constraint holds almost everywhere in~$\Omega$ is omitted for conciseness. Thus, the dual problem of~\eqref{P:generalSFP} is given by
\begin{prob}[D-SFP]\label{P:dualSFP}
	\maximize_{\bmu,\ \bnu \succeq 0}& &&d(\bmu,\bnu)
		\text{.}
\end{prob}

By definition, \eqref{P:dualSFP} is a convex program whose dimensionality is equal to the number of constraints~\cite{Boyd04c}---in this case, on the order of~$p$. It is therefore tractable as long as we can evaluate the dual function~$d$. Indeed, solving~\eqref{P:dualSFP} is at least as hard as solving the minimization in~\eqref{E:preDualFunction1}. We next show that the dual function of SFPs is often efficiently computable.

The joint minimization in~\eqref{E:preDualFunction1} separates as
\begin{equation}\label{E:preDualFunction2}
\begin{aligned}
	d(\bmu, \bnu) &= d_X(\bmu) + d_{\bz}(\bmu,\bnu)
\end{aligned}
\end{equation}
with
\begin{equation}\label{E:dX}
\begin{aligned}
	d_X(\bmu) &= \min_{\substack{X \in \calX, \\ X(\bbeta) \in \calP}}
		\int_{\Omega} \left\{ F_0 \left[ X(\bbeta),\bbeta \right]
		+ \lambda \indicator\left[ X(\bbeta) \neq 0 \right]
		\vphantom{\sum}\right.
	\\
	&\qquad\qquad\qquad \left.\vphantom{\sum}
	{}+ \Re \left[ \bmu^H \bF \left[ X(\bbeta), \bbeta \right] \right] \right\} d\bbeta
\end{aligned}
\end{equation}
and~$d_{\bz}(\bmu,\bnu) = \min_{\bz} \sum_{i = 1}^m \nu_i g_i(\bz) - \Re[ \bmu^H \bz ]$. The minimum in~$d_{\bz}$ is tractable since the objective is convex, given that~$\nu_i \geq 0$ and the~$g_i$ are convex functions. In certain cases, e.g., when~$g_i$ is a quadratic loss, $d_{\bz}$ may even have a closed-form expression. On the other hand, $d_X$ is in general a non-convex problem. When~$F_0$ and~$\bF$ are normal integrands~\cite[Def.~14.27]{Rockafellar98v}, this issue is addressed by exploiting the separability of the objective across~$\bbeta$ as shown in Proposition~\ref{T:dualXSolution}. Examples of normal integrands include functions~$f( x,\bbeta )$ that are continuous in~$x$ for all fixed~$\bbeta$ and measurable in~$\bbeta$ for all fixed~$x$~(also known as Carath\'{e}odory) or when~$\Omega$ is Borel and~$f( \cdot,\bbeta )$ is lower semicontinuous for all fixed~$\bbeta$~\cite{Rockafellar98v}. Note that these functions can be nonlinear and need not be convex.

\begin{proposition}\label{T:dualXSolution}

Consider the functional optimization problem in~\eqref{E:dX} and assume that~$F_0$ and the elements of~$\bF$ are normal integrands. Let~$\gamma^{(0)}(\bmu,\bbeta) = F_0(0,\bbeta) + \Re \left[ \bmu^H \bF(0, \bbeta) \right]$ and define
\begin{equation}\label{E:gammao}
	\gamma^o(\bmu,\bbeta) = \min_{x \in \calP} F_0(x,\bbeta) + \Re \left[ \bmu^H \bF(x, \bbeta) \right]
		\text{.}
\end{equation}
Then, for~$\calS(\bmu) = \{\bbeta \in \Omega : \gamma^o(\bmu,\bbeta) < \gamma^{(0)}(\bmu,\bbeta) - \lambda\}$,
\begin{equation}\label{E:dX_value}
	d_X(\bmu) = \int_{\calS(\bmu)} \left[ \lambda + \gamma^o(\bmu,\bbeta) \right] d\bbeta
		+ \int_{\Omega \setminus \calS(\bmu)} \gamma^{(0)}(\bmu,\bbeta) d\bbeta
		\text{.}
\end{equation}
\end{proposition}

\begin{proof}

We start by separating the objective of~\eqref{E:dX} using the following lemma:

\begin{lemma}[{Separability principle~\cite[Thm.~14.60]{Rockafellar98v}}]\label{T:minInt}
	Let~$G(x,\bbeta)$ be a normal integrand and~$\calX$ be a decomposable space. Then,
	\begin{equation}\label{E:exchangeInf}
		\inf_{\substack{X \in \calX \\ X(\bbeta) \in \calP}} \int_\Omega
			G\left[ X(\bbeta), \bbeta \right] d\bbeta =
			\int_\Omega \inf_{x \in \calP} G(x,\bbeta) d\bbeta
			\text{.}
	\end{equation}
\end{lemma}

\noindent Since~$\Omega$ is a compact subset of~$\setR^n$ and the indicator function is lower semicontinuous~\cite[Ex.~14.31]{Rockafellar98v}, the integrand in~\eqref{E:dX} is normal and we can restrict ourselves to solving the optimization individually for each~$\bbeta$, i.e., if~$X_d$ is a solution of~\eqref{E:dX}, then
\begin{prob}\label{P:dualMinimizer}
	X_d(\bbeta) \in \argmin_{x \in \calP} F_0 \left[ x,\bbeta \right]
		+ \lambda \indicator\left[ x \neq 0 \right]
		+ \Re \left[ \bmu^H \bF \left[ x, \bbeta \right] \right]
		\text{.}
\end{prob}
Despite the non-convexity of the indicator function, \eqref{P:dualMinimizer} is a scalar problem, whose solution involves a simple thresholding scheme. Indeed, only two conditions need to be checked: (i)~if~$X_d(\bbeta) = 0$, then the indicator function vanishes and the objective of~\eqref{P:dualMinimizer} evaluates to~$\gamma^{(0)}(\bbeta)$; (ii)~if~$X_d(\bbeta) \neq 0$, then the indicator function is one and the objective of~\eqref{P:dualMinimizer} evaluates to~$\lambda + \gamma^o(\bbeta)$. The value of~\eqref{P:dualMinimizer} is the minimum of these two cases, which from Lemma~\ref{T:minInt} yields the desired result in~\eqref{E:dX_value}.
\end{proof}

Proposition~\ref{T:dualXSolution} provides a practical way to evaluate~\eqref{E:preDualFunction2}, i.e., to evalute the objective of the dual problem~\eqref{P:dualSFP}. Still, it relies on the ability to efficiently solve~\eqref{E:gammao}, which may be an issue if~$F_0$, $\bF$, or~$\calP$ are non-convex. Nevertheless, \eqref{E:gammao} remains a scalar problem that can typically be solved efficiently using global optimization techniques~\cite{Hendrix10i} or through efficient local search procedures~(see Section~\ref{S:Apps}).

The tractability of the dual problem~\eqref{P:dualSFP} does not imply that it provides a solution to the original problem~\eqref{P:generalSFP}. In fact, since SFPs are not convex programs it is not immediate that~\eqref{P:dualSFP} is worth solving at all: there is no reason to expect that the optimal value of~\eqref{P:dualSFP} is anything more than a lower bound on the optimal value of~\eqref{P:generalSFP}~\cite{Boyd04c}. In the sequel, we proceed to show that this is not the case and that we can actually obtain a solution of~\eqref{P:generalSFP} by solving~\eqref{P:dualSFP}.

\section{Strong duality and its implications}
	\label{S:dualityGap}

Though we have argued that the dual problem of~\eqref{P:generalSFP} is potentially tractable, we are ultimately interested in solving~\eqref{P:generalSFP} itself. This section tackles this limitation by showing that~\eqref{P:generalSFP} and~\eqref{P:dualSFP} have the same values~(Theorem~\ref{T:strongDuality}). In Section~\ref{S:dualSolution}, we show how, under mild conditions, this result allows us to efficiently find a solution for~\eqref{P:generalSFP}. Before that, however, we use strong duality to derive a relation between SFPs and~$L_1$-norm optimization problems when their solution saturates~(Section~\ref{S:L1}).

\subsection{Strong duality of sparse functional programs}

The main result of this section is presented in the following theorem:

\begin{theorem}\label{T:strongDuality}

Suppose that~$F_0$ and~$\bF$ have no point masses~(Dirac deltas) and that there exists a~\eqref{P:generalSFP}-feasible pair~$(X^\prime,\bz^\prime)$, i.e., $X^\prime \in \calX$ with~$X^\prime(\bbeta) \in \calP$ a.e.\ and~$\bz^\prime \in \setC^p$, such that~$\bz^\prime = \int_{\Omega} \bF \left[ X^\prime(\bbeta), \bbeta \right] d\bbeta$ and~$g_i(\bz^\prime) < 0$ for all~$i = 1,\dots,m$. Then, strong duality holds for~\eqref{P:generalSFP}, i.e., if~$P^\star$ is the optimal value of~\eqref{P:generalSFP} and~$D^\star$ is the optimal value of~\eqref{P:dualSFP}, then~$P^\star = D^\star$.

\end{theorem}

Theorem~\ref{T:strongDuality} states that although~\eqref{P:generalSFP} is a non-convex functional program, it has zero duality gap, suggesting that it can be solved through its tractable dual~\eqref{P:dualSFP}. A noteworthy feature of this approach is that it precludes discretization by tackling~\eqref{P:generalSFP} directly. Discretizing~\eqref{P:generalSFP} may not only result in NP-hard problems, but leads to high dimensional, potentially ill-conditioned problems. It is also worth noting that Theorem~\ref{T:strongDuality} is \emph{non-parametric} in the sense that it makes no assumption on the existence or validity of the measurement model in~\eqref{P:generalSFP}. In particular, it does not require that the data arise from a specific model in which the parameters are sparse. This implies, for instance, that the sparsest functional linear model that fits a set of measurements can be determined regardless of whether these measurements arise from a truly sparse, linear model. This is useful in practice when sparse solutions are sought, not for epistemological reasons, but for reducing computational or measurement costs.

\begin{proof}[Proof of Thm.~\ref{T:strongDuality}]

Recall from weak duality that the dual problem is a lower bound on the value of the primal, so that~$D^\star \leq P^\star$~\cite{Boyd04c}. Hence, it suffices to prove that~$D^\star \geq P^\star$. To do so, denote the cost function of~\eqref{P:generalSFP} by~$f_0(X) \triangleq \int_{\Omega} F_0 \left[ X(\bbeta),\bbeta \right] d\bbeta + \lambda \norm{X}_{L_0}$ and define the cost-constraints set
\begin{equation}\label{E:cost_constraint}
\begin{gathered}
	\calC = \left\{ (c,\bu,\bk_R,\bk_I) \ \Big\vert \  \exists (X,\bz) \in \calX \times \setC^p
	\text{ such that } 
	\vphantom{\int_{\Omega}}\right.\\\left.
	X(\bbeta) \in \calP \text{ a.e.} \text{, }
		f_0(X) \leq c \text{, }
		g_i(\bz) \leq [\bu]_i \text{,}
	\vphantom{\int_{\Omega}}\right.\\\left.
	\text{and } \int_{\Omega} \bF \left[ X(\bbeta), \bbeta \right] d\bbeta - \bz = \bk_R + j \bk_I
	\right\}
		\text{.}
\end{gathered}
\end{equation}
In words, $\calC$ describes the range of values taken by the objective and constraints of~\eqref{P:generalSFP}. Observe that~\eqref{E:cost_constraint} separates the real- and complex-valued parts of the equality constraint in~\eqref{P:generalSFP}. Hence, $\calC \subset \setR^{2p+m+1}$, allowing us to directly leverage classical convex geometry results. The crux of this proof is to show that~$\calC$ is a convex set even though~\eqref{P:generalSFP} is not a convex program. We summarize this result in the following technical lemma whose proof relies on Lyapunov's convexity theorem~\cite{Diestel77v}:

\begin{lemma}\label{T:convexC}

Under the assumptions of Theorem~\ref{T:strongDuality}, the cost-constraints set~$\calC$ in~\eqref{E:cost_constraint} is a non-empty convex set.

\end{lemma}

\begin{proof}
See Appendix~\ref{A:convexC}.
\end{proof}

We may then leverage the following result from convex geometry:

\begin{proposition}[{Supporting hyperplane theorem~\cite[Prop.~1.5.1]{Bertsekas09c}}]\label{T:hyperplane}
Let~$\calA \subset \setR^n$ be a nonempty convex set. If~$\bxt \in \setR^n$ is not in the interior of~$\calA$, then there exists a hyperplane passing through~$\bxt$ such that~$\calA$ is in one of its closed halfspaces, i.e., there exists~$\bp \neq \bzero$ such that~$\bp^T \bxt \leq \bp^T \bx$ for all~$\bx \in \calA$.

\end{proposition}

Formally, start by observing that the point~$(P^\star,\bzero,\bzero,\bzero)$ cannot be in the interior of~$\calC$. Indeed, there would otherwise exist~$\delta > 0$ such that~$(P^\star-\delta,\bzero,\bzero,\bzero) \in \calC$, violating the optimality of~$P$. Proposition~\ref{T:hyperplane} therefore implies that there exists a non-zero vector~$(\lambda_0,\bnu,\bmu_R,\bmu_I) \in \setR^{2p+m+1}$ such that for all~$(c,\bu,\bk_R,\bk_I) \in \calC$,
\begin{equation}\label{E:hyperplane_bound}
	\lambda_0 c + \bnu^T \bu + \bmu_R^T \bk_R + \bmu_I^T \bk_I \geq \lambda_0 P^\star
		\text{.}
\end{equation}
Observe that the vector defining the hyperplane uses the same notation as for the dual variables of~\eqref{P:generalSFP} foreshadowing the fact that these hyperplanes span the values of the Lagrangian~\eqref{E:lagrangian}. From~\eqref{E:hyperplane_bound}, we immediately obtain that~$\lambda_0 \geq 0$ and~$\bnu \succeq \bzero$. Indeed, note that~$\calC$ is unbounded above in its first~$m + 1$ components, i.e., if~$(c,\bu,\bk_R,\bk_I) \in \calC$ then~$(c^\prime,\bu^\prime,\bk_R,\bk_I) \in \calC$ for any~$(c^\prime,\bu^\prime) \succeq (c,\bu)$. Hence, if any component of~$\lambda_0$ or~$\bnu$ were negative, there would exist a vector in~$\calC$ that makes the left-hand side of~\eqref{E:hyperplane_bound} arbitrarily small, eventually violating the inequality. Let us now show that~$\lambda_0 \neq 0$.

To do so, suppose~$\lambda_0 = 0$. Then, \eqref{E:hyperplane_bound} reduces to
\begin{equation}\label{E:lambda0_inequality}
	\bnu^T \bu + \bmu_R^T \bk_R + \bmu_I^T \bk_I \geq 0
		\text{,}
\end{equation}
for all~$(c,\bu,\bk_R,\bk_I) \in \calC$. However, \eqref{E:lambda0_inequality} leads to a contradiction because its left-hand side can always be made negative. Indeed, if~$[\bnu]_i > 0$ for any~$i$, then the hypothesis on the existence of a strictly feasible point for~\eqref{P:generalSFP} violates~\eqref{E:lambda0_inequality}. Explicitly, since there exists~$(X^\prime,\bz^\prime)$ such that~$\int_{\Omega} \bF \left[ X^\prime(\bbeta), \bbeta \right] d\bbeta = \bz^\prime$ and~$g_i(\bz^\prime) < 0$ for all~$i = 1,\dots,m$, then~$(c_0,-\delta \ones,\bzero,\bzero) \in \calC$ for some~$c_0$ and~$\delta > 0$, where~$\ones$ is a vector of ones. Thus, if~$[\bnu]_i > 0$ for any~$i$, we obtain~$-\delta (\bnu^T \ones) < 0$, which violates~\eqref{E:lambda0_inequality}.
	
On the other hand, if~$\bnu = \bzero$, then~\eqref{E:lambda0_inequality} reduces to~$\bmu_R^T \bk_R + \bmu_I^T \bk_I \geq 0$ which cannot holds because for~$(\bar{c},\bar{\bu},-\bmu_R,-\bmu_I) \in \calC$, \eqref{E:lambda0_inequality} evaluates to~$-\norm{\bmu_R}^2 -\norm{\bmu_I}^2 < 0$. To see that this vector is indeed an element of~$\calC$, simply choose any~$\bar{X} \in \calC$ with~$\bar{X}(\beta) \in \calP$~a.e.\ and let~$\bar{\bz} = - \bmu_R - j \bmu_I - \int_{\Omega} \bF \left[ \bar{X}(\bbeta), \bbeta \right] d\bbeta$, $[\bar{\bu}]_i = g_i(\bar{\bz})$, and~$\bar{c} = f_0(\bar{X})$. Hence, it must be that~$\lambda_0 \neq 0$.

However, for~$\lambda_0 \neq 0$, \eqref{E:hyperplane_bound}
\begin{equation*}
	c + \bm{{\tilde{\nu}}}^T \bu + \bm{{\tilde{\mu}}}_R^T \bk_R + \bm{{\tilde{\mu}}}_I^T \bk_I
		\geq P
		\text{,}
\end{equation*}
where~$\bm{{\tilde{\nu}}} = \bnu/\lambda_0$, $\bm{{\tilde{\mu}}}_R = \bmu_R/\lambda_0$, and~$\bm{{\tilde{\mu}}}_I = \bmu_I/\lambda_0$, which from the definition of~$\calC$ implies that
\begin{multline}\label{E:case_ii}
	f_0(X) + \sum_{i = 1}^m \tilde{\nu}_i g_i(\bz)
	+ \bm{{\tilde{\mu}}}_R^T \Re \left[ \int_{\Omega} \bF \left[ X(\bbeta), \bbeta \right] d\bbeta
		- \bz \right]
	\\
	+ \bm{{\tilde{\mu}}}_I^T \Im \left[ \int_{\Omega} \bF \left[ X(\bbeta), \bbeta \right] d\bbeta
		- \bz \right] \geq P^\star
		\text{,}
\end{multline}
for any~\eqref{P:generalSFP}-feasible pair~$(X,\bz)$. Letting~$\bm{{\tilde{\mu}}} = \bm{{\tilde{\mu}}}_R + j \bm{{\tilde{\mu}}}_I$, we recognize that~\eqref{E:case_ii} in fact bounds the value of the Lagrangian in~\eqref{E:lagrangian} for any~\eqref{P:generalSFP}-feasible pair~$(X,\bz)$, i.e., $\calL(X, \bz, \bm{{\tilde{\mu}}}, \bm{{\tilde{\nu}}}) \geq P^\star$. Taking the minimum of the left-hand side of~\eqref{E:case_ii} hence implies~$D^\star \geq P^\star$, thus concluding the proof.
\end{proof}

\subsection{SFPs and $L_1$-norm optimization problems}
	\label{S:L1}

Similar to the discrete case, there is a close relation between~$L_0$- and~$L_1$-norm minimization. Formally, consider
\begin{prob}[$\text{P}_q$]\label{P:L0L1}
	\minimize_{X \in L_\infty,\, \bz \in \calC^p}& &&\norm{X}_{L_q}
	\\
	\subjectto& &&g_i(\bz) \leq 0
	\\
	&&&\bz = \int_{\Omega} \bF [ X(\bbeta), \bbeta ] d\bbeta
	\\
	&&&\abs{X} \leq \Gamma \text{ a.e.}
\end{prob}
Problem~($\text{P}_0$)~[i.e., \eqref{P:L0L1} with~$q = 0$] is an instance of~\eqref{P:generalSFP} without regularization~($F_0 \equiv 0$) in which~$\calP$ is the set of measurable functions bounded by~$\Gamma > 0$. On the other hand, ($\text{P}_1$)~[\eqref{P:L0L1} for~$q = 1$] is a functional version of the classical~$\ell_1$-norm minimization problem. The following proposition shows that for a wide class of dictionaries, the optimal values of~($\text{P}_0$) and~($\text{P}_1$) are the same~(up to a constant).

\begin{proposition}\label{T:L0L1}

Let~$x^o(\bmu, \bbeta) = \argmin_{\abs{x} \leq \Gamma} \abs{x} + \Re\left[ \bmu^T \bF(x,\bbeta) \right]$ saturate, i.e., $x^o(\bmu, \bbeta) \neq 0 \Rightarrow \abs{x^o(\bmu, \bbeta)} = \Gamma$ for all~$\bmu \in \setC^p$ and~$\bbeta \in \Omega$. If~$P_0^\star$~($P_1^\star$) is the optimal value of~\eqref{P:L0L1} for~$q = 0$~($q = 1$) and Slater's condition holds, then
\begin{equation*}
	P_0^\star = \frac{P_1^\star}{\Gamma}
		\text{.}
\end{equation*}

\end{proposition}

\begin{proof}

The proof follows by relating the dual values of~\eqref{P:L0L1} for~$q = \{0,1\}$ and then using strong duality. Start by defining the Lagrangian of~\eqref{P:L0L1} as
\begin{equation}\label{E:lagrangianL0L1}
\begin{aligned}
	\calL(X, \bz, \bmu, \bnu) &= \norm{X}_{L_q}
		+ \sum_i \nu_i g_i(\bz)
	\\
	{}&+ \Re \left[
		\bmu^H \left( \int_{\Omega} \bF\left[ X(\bbeta),\bbeta \right] d\bbeta - \bz \right)
	\right]
		\text{.}
\end{aligned}
\end{equation}
For~$q = 0$, Proposition~\ref{T:minInt} yields
\begin{equation}\label{E:d0}
\begin{aligned}
	d_0(\bmu, \nu) &= \int_{\calS_0(\bmu)}
		\left\{ 1 + \min_{\abs{x} \leq \Gamma} \Re\left[ \bmu^H \bF(x,\bbeta) \right] \right\} d\bbeta
	\\
	{}&+ w(\bmu,\bnu)
		\text{.}
\end{aligned}
\end{equation}
where
\begin{equation}\label{E:calS0}
	\calS_0(\bmu) = \{\bbeta \in \Omega \mid
		\min_{\abs{x} \leq \Gamma} \Re\left[ \bmu^H \bF(x,\bbeta) \right] < -1 \}
\end{equation}
and~$w(\bmu,\bnu) = \min_{\bz} \sum_i \nu_i g_i(\bz) - \Re\left[ \bmu^H \bz \right]$. Notice that~$w$ is homogeneous, i.e., $w(\alpha\bmu, \alpha\bnu) = \alpha w(\bmu,\bnu)$ for~$\alpha > 0$. Proceeding similarly from~\eqref{E:lagrangianL0L1}, the dual function of~($\text{P}_1$) is
\begin{equation}\label{E:pred1}
\begin{aligned}
	d_1(\bmu, \nu) &= \int_{\Omega}
		\left\{ \min_{\abs{x} \leq \Gamma} \abs{x} + \Re\left[ \bmu^H \bF(x,\bbeta) \right] \right\}
	d\bbeta
	\\
	{}&+ w(\bmu,\bnu)
		\text{.}
\end{aligned}
\end{equation}
Using the the saturation hypothesis, the integrand in~\eqref{E:pred1} is non-trivial only over the set
\begin{equation}\label{E:calS1}
	\calS_1(\bmu) = \{\bbeta \in \Omega \mid
		\min_{\abs{x} \leq \Gamma} \Re\left[ \bmu^H \bF(x,\bbeta) \right] < -\Gamma \}
		\text{.}
\end{equation}
Hence,
\begin{equation}\label{E:d1}
\begin{aligned}
	d_1(\bmu, \nu) &= \int_{\calS_1(\bmu)}
		\left\{ \Gamma + \min_{\abs{x} \leq \Gamma} \Re\left[ \bmu^H \bF(x,\bbeta) \right] \right\}
	d\bbeta
	\\
	{}&+ w(\bmu,\bnu)
		\text{.}
\end{aligned}
\end{equation}

To proceed, note that the dual functions in~\eqref{E:d0} and~\eqref{E:d1} are related by
\begin{equation}\label{E:relationL0L1}
	d_0(\bmu,\bnu) = \frac{1}{\Gamma} d_1(\Gamma\bmu, \Gamma\bnu)
		\text{.}
\end{equation}
Indeed, observe from~\eqref{E:calS0} and~\eqref{E:calS1} that~$\calS_1(\Gamma \bmu) = \calS_0(\bmu)$. Thus,
\begin{multline*}
	\frac{1}{\Gamma} \int_{\calS_1(\Gamma \bmu)} \left\{
		\Gamma + \min_{\abs{x} \leq \Gamma} \Re\left[ \Gamma \bmu^H \bF(x,\bbeta) \right]
	\right\}
	d\bbeta
	\\
	= \int_{\calS_0(\bmu)}
		\left\{ 1 + \min_{\abs{x} \leq \Gamma} \Re\left[ \bmu^H \bF(x,\bbeta) \right] \right\}
	d\bbeta
		\text{.}
\end{multline*}
The homogeneity of~$w$ then yields~\eqref{E:relationL0L1}. Immediately, it holds that if~$(\bmu^o,\bnu^o)$ is a maximum of~$d_0$, then $(\Gamma\bmu^o,\Gamma\bnu^o)$ is a maximum of~$d_1$. To see this is the case, note from~\eqref{E:relationL0L1} that
\begin{equation*}
	\nabla d_0(\bmu^o,\bnu^o) = \bzero \Leftrightarrow
		\nabla d_1(\Gamma\bmu^o, \Gamma\bnu^o) = \bzero
		\text{,}
\end{equation*}
so that~$(\Gamma\bmu^o, \Gamma\bnu^o)$ is a critical point of~$d_1$. Since~$d_1$ is a concave function, $(\Gamma\bmu^o, \Gamma\bnu^o)$ must be a global maximum.

To conclude, observe that~\eqref{P:L0L1} has zero duality gap for both~$q = 0$, due to Theorem~\ref{T:strongDuality}, and~$q = 1$, because it is a convex program. From~\eqref{E:relationL0L1} we then obtain
\begin{align*}
	P_0^\star &= \max_{\bmu, \bnu \geq 0} d_0(\bmu,\bnu) = d_0(\bmu^\star,\bnu^\star)
	\\
	{}&= \frac{1}{\Gamma} d_1(\Gamma\bmu^\star, \Gamma\bnu^\star)
		= \frac{1}{\Gamma} \max_{\bmu, \bnu \geq 0} d_1(\bmu,\bnu) = \frac{P_1^\star}{\Gamma}
		\text{.}\qedhere
\end{align*}
\end{proof}

Proposition~\ref{T:L0L1} shows that a large class of~$L_0$- and~$L_1$-norm minimization problems found in functional nonlinear sparse recovery are equivalent in the sense that their optimal values are~(essentially) the same. It is worth noting that establishing this relation requires virtually no assumptions: the saturation hypothesis is met by a wide class of dictionaries, most notably linear ones. This is in contrast to the discrete case, where such relations exist for incoherent, linear dictionaries~\cite{Eldar12c, Foucart13m}. Still, Proposition~\ref{T:L0L1} does not imply that the solution of the~$L_0$- and~$L_1$-norm problems are the same, as is the case for discrete results. In fact, though they have the same optimal value, ($\text{P}_1$) admits solutions with larger support~(see Example~\ref{R:L0L1}). Although conditions exist for which the~$L_1$-norm minimization problem with linear dictionaries yields minimum support solutions~\cite{Adcock16g, Adcock17b, Puy17r}, Theorem~\ref{T:strongDuality} precludes the use of this relaxation for both linear and nonlinear dictionaries.

\begin{figure}[tb]
\centering
\includesvg{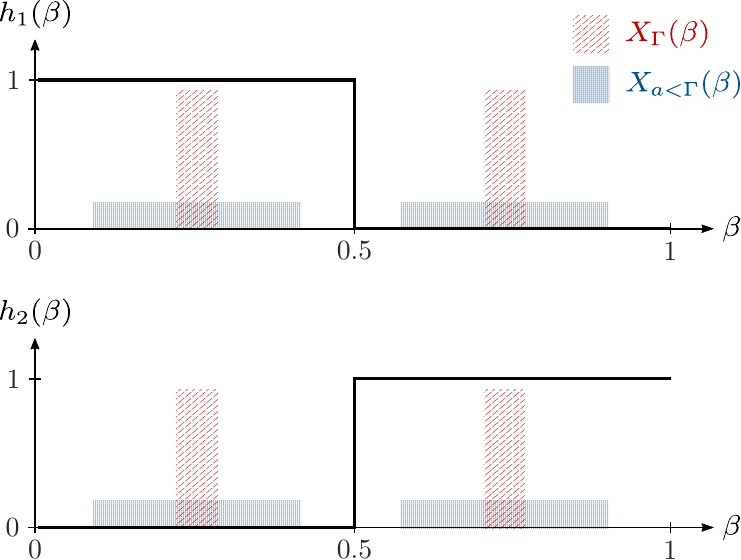}
\caption{Illustration of Example~\ref{R:L0L1}.}
	\label{F:example}
\end{figure}

\begin{example}\label{R:L0L1}

Proposition~\ref{T:L0L1} gives an equivalence between~$L_0$- and~$L_1$-norm minimization problems in terms of their optimal values, but not in terms of their solutions. We illustrate this point using the example depicted in Figure~\ref{F:example}. Let~$\Omega = [0,1]$, $g(\bz) = \norm{\by - \bz}_2^2$, $\by = \vect{cc}{y_1 & y_2}^T$ with~$\abs{y_1},\abs{y_2} < \Gamma/2$, and~$\bF(x,\beta) = \bh(\beta) x$, where~$\bh(\beta) = \vect{cc}{h^\prime(\beta) & 1-h^\prime(\beta)}^T$ with~$h^\prime(\beta) = \indicator(\beta \in [0,1/2])$. Due to the form of~$h^\prime$, it is readily seen that the optimal value of~($\text{P}_1$) is~$P_1^\star = \abs{y_1} + \abs{y_2}$.

Now consider the family of functions indexed by~$0 < a \leq \Gamma$
\begin{equation}
	X_a(\beta) = a \sign(y_1) \indicator(\beta \in \calA_1)
		+ a \sign(y_2) \indicator(\beta \in \calA_2)
		\text{,}
\end{equation}
where~$\calA_1 \subseteq [0,1/2]$ with~$\norm{\calA_1}_{L_0} = \abs{y_1}/a$ and~$\calA_2 \subseteq [1/2,1]$ with~$\norm{\calA_2}_{L_0} = \abs{y_2}/a$~(e.g., $X_\Gamma$ and~$X_{a < \Gamma}$ in Figure~\ref{F:example}). For all~$a$, $X_a$ is a solution of~($\text{P}_1$)~[it is~\eqref{P:L0L1}-feasible with value~$P_1^\star$]. However, its support is given by~$\norm{X_a}_{L_0} = (\abs{y_1}+\abs{y_2})/a$. Thus, ($\text{P}_1$) admits solutions that do not have minimum support~($X_{a < \Gamma}$ in Figure~\ref{F:example}), whereas only~$X_\Gamma$ is a solution of~($\text{P}_0$).

\end{example}

\section{Solving sparse functional programs}
	\label{S:dualSolution}

Theorem~\ref{T:strongDuality} from the previous section establishes duality as a fruitful approach for solving the sparse functional program~\eqref{P:generalSFP}. Indeed, the strong duality of~\eqref{P:generalSFP} implies that
\begin{equation}\label{E:primalFromDual}
	(X^\star, \bz^\star) \in
		\argmin_{\substack{X \in \calX,\, \bz \in \setC^p, \\ X(\bbeta) \in \calP}}
			\calL(X, \bz, \bmu^\star, \bnu^\star)
		\text{,}
\end{equation}
for the Lagrangian~$\calL$ in~\eqref{E:lagrangian}, where~$\bmu^\star$ and~$\bnu^\star$ are the solutions of~\eqref{P:dualSFP}~\cite{Boyd04c}. When this set is a singleton, the inclusion becomes equality and we recover the unique primal solution~$X^\star$. This occurs when the Lagragian~\eqref{E:lagrangian} has a single minimizer, i.e., when~\eqref{P:dualMinimizer} is a singleton. This is the case, for instance, when~$F_0(x,\bbeta) = \abs{x}^2$, in which case~$\calL$ is strongly convex in~$X$~\cite{Boyd04c}. Since Proposition~\ref{T:dualXSolution} allows us to solve~\eqref{E:primalFromDual}, all that remains is to address the issue of solving~\eqref{P:dualSFP} to obtain~$(\bmu^\star,\bnu^\star)$.

Note that~\eqref{P:dualSFP} is a convex program and can therefore be solved using any~(stochastic) convex optimization algorithm~\cite{Boyd04c, Bertsekas15c}. For illustration, this section introduces an algorithm based on supergradient ascent. For ease of reference, a step-by-step guide to solving SFPs is presented in Appendix~\ref{A:solvingSFPs}.

Recall that a supergradient of a concave function~$f: \Omega \to \setR$ at~$\bx \in \Omega$ is any vector~$\bp$ that satisfies the inequality~$f(\by) \leq f(\bx) + \bp^T (\by - \bx)$ for all~$\by \in \Omega$. Although a supergradient may not be an ascent direction at~$\bx$, taking small steps in its direction decreases the distance to any maximizer of~$f$~\cite{Boyd04c}.

It is straightforward to show that the constraint slacks in~\eqref{E:lagrangian} are supergradients of the dual function~$d$ with respect to their corresponding dual variables~\cite{Boyd04c}. Explicitly,
\begin{subequations}\label{E:grad}
\begin{align}
	\bp_{\bmu}(\bmu^\prime,\bnu^\prime) &=
		\int_{\Omega} \bF \left[ X_d(\bmu^\prime, \bbeta), \bbeta \right] d\bbeta
		- \bz_d(\bmu^\prime,\bnu^\prime)
		\label{E:gradMu}
	\\
	\bp_{\nu_i}(\bmu^\prime,\bnu^\prime) &= g_i\left[ \bz_d(\bmu^\prime,\bnu^\prime) \right]
		\label{E:gradNu}
\end{align}
\end{subequations}
are supergradients of~$d$ for the dual minimizers
\begin{subequations}\label{E:dualMinimizers}
\begin{align}
	X_d(\bmu,\cdot) &\in \argmin_{\substack{X \in \calX \\ X(\bbeta) \in \calP}}
		\int_{\Omega} \left\{ F_0 \left[ X(\bbeta),\bbeta \right]
		+ \lambda \indicator\left[ X(\bbeta) \neq 0 \right]
		\vphantom{\sum}\right.
	\notag\\
	&\left.\vphantom{\sum}
	{}+ \Re \left[ \bmu^H \bF \left[ X(\bbeta), \bbeta \right] \right] \right\} d\bbeta
		\text{,}
		\label{E:Xd}
	\\
	\bz_d(\bmu,\bnu) &\in \argmin_{\bz} \sum_i \nu_i g_i(\bz) - \Re[ \bmu^H \bz ]
		\text{.}
		\label{E:zd}
\end{align}
\end{subequations}
Algorithm~\ref{L:dualAscent}, with step size~$\eta_t > 0$, then yields the optimal dual variables~$(\bmu^\star, \bnu^\star)$ and a solution~$X^\star$ of~\eqref{P:generalSFP}.

\begin{algorithm}[t]
\centering
\caption{Dual ascent for SFPs}
	\label{L:dualAscent}
	\setlength{\baselineskip}{1.2\baselineskip}
%	\small
	\begin{algorithmic}
		\State $\bmu^{(0)} = \bzero$, $\nu_i^{(0)} = 1$

		\For{$t = 1,\dots,T$}
		
			\State $\displaystyle
				X_{t-1}(\bbeta) = X_d\left( \bmu^{(t-1)}, \bbeta \right)$
				
			\State $\displaystyle
				\bz_{t-1} = \bz_d\left( \bmu^{(t-1)},\bnu^{(t-1)} \right)$

			\State $\displaystyle
				\bmu^{(t)} = \bmu^{(t-1)} + \eta_t
				\left[ \int_\Omega \bF \left[ X_{t-1}(\bbeta), \bbeta \right] d\bbeta - \bz_{t-1} \right]$

			\State $\displaystyle
				\nu_i^{(t)} = \left[ \nu_i^{(t-1)} + \eta_t g_i\left( \bz_{t-1} \right) \right]_+$
				
			\State $\displaystyle
				P_t = d(\bmu^{(t)}, \bnu^{(t)})$

		\EndFor
		
		\State $X^\star(\bbeta) = X_d\left( \bmu^{(t^\star)}, \bbeta \right)$ for $\displaystyle
			t^\star \in \argmax_{1 \leq t \leq T} P_t$
	\end{algorithmic}
\end{algorithm}

Given that the optimization problem in~\eqref{E:zd} is convex, there are two hurdles in evaluating~\eqref{E:grad}: (i)~obtaining~$X_d$ involves solving the non-convex, infinite dimensional problem in~\eqref{E:Xd} and (ii)~the integral in~\eqref{E:gradMu} may not have an explicit form or this form is too cumbersome to be useful in practice. We have already argued that despite its non-convexity, the minimization in~\eqref{E:Xd} is tractable by exploiting separability~(see Proposition~\ref{T:dualXSolution}). The resulting scalar problem often has a closed-form solution~(see Section~\ref{S:Apps} for examples) or can be tackled using global optimization techniques~\cite{Hendrix10i}. Note that though this approach does not explicitly yield the function~$X_d$, it allows~$X_d(\bmu,\bbeta)$ to be evaluated for any~$\bbeta \in \Omega$ using~\eqref{P:dualMinimizer}. This is enough to numerically compute the integral in~\eqref{E:gradMu}. This integral~[(ii)] may either be approximated numerically or done without by leveraging stochastic optimization techniques. Step-by-step descriptions of both methods are presented in Appendix~\ref{A:solvingSFPs}.

In the first case, we effectively solve a perturbed version of~\eqref{P:generalSFP} and the difference between the optimal value of the original problem and that obtained numerically depends linearly on the precision of the integral computation under mild technical conditions:

\begin{proposition}\label{T:numericalIntegration}

Suppose that
\begin{enumerate}[(i)]

\item the perturbation function of~\eqref{P:generalSFP} is differentiable around the origin;
\item $\int_\Omega F_0(0,\bbeta) d\bbeta = 0$ and $\int_\Omega \bF(0,\bbeta) d\bbeta = \bzero$;
\item there exists~$\alpha > 0$ such that~$g_i(\alpha \ones), g_i(- \alpha \ones) < \infty$; and
\item there exists a strictly feasible pair~$(X^\dagger,\bz^\dagger)$~(Slater's condition) for~\eqref{P:generalSFP} such that~$g_i(\bz^\dagger) < -\epsilon$, for~$\epsilon > 0$, and~$\bar{F}_0 = \int_\Omega F_0(X^\dagger(\bbeta),\bbeta) d\bbeta < \infty$.

\end{enumerate}
If~$P^\star$ is the optimal value of~\eqref{P:generalSFP} and~$P_\delta^\star$ is the value of the solution obtained by Algorithm~\ref{L:dualAscent} when evaluating the integral in~\eqref{E:gradMu} with approximation error~$0 < \delta \ll 1$, then~$\abs{P^\star - P_\delta^\star} \leq \calO(\delta)$.

\end{proposition}

\begin{proof}

See Appendix~\ref{A:numericalIntegration}.
\end{proof}

In the second case, the integral in~\eqref{E:gradMu} is approximated using Monte Carlo integration, i.e., by drawing a set of~$\bbeta_j$ independently and uniformly at random from~$\Omega$ and taking
\begin{equation}\label{E:mcH}
	\bph_{\bmu} = \frac{1}{N} \sum_{j = 1}^N \bF \left[ X_d(\bmu^\prime, \bbeta_j), \bbeta_j \right]
		- \bz_d(\bmu^\prime,\bnu^\prime)
		\text{.}
\end{equation}
Since Monte Carlo integration is an unbiased estimators, $\bph_{\bmu}$ is an unbiased estimate of~$\bp_{\bmu}$. Taking~$N = 1$ in~\eqref{E:mcH} is akin to performing stochastic (super)gradient ascent on the dual function~$d$. For~$N > 1$, we obtain a mini-batch type algorithm. Typical convergence guarantees hold in both cases~\cite{Ruszczynski86c, Boyd04c, Bottou16o}.

Algorithm~\ref{L:dualAscent}, though effective, may converge slowly depending on the numerical properties of the problem. Faster, problem independent convergence rates can be obtained using, for instance, second-order methods or by exploiting specific structures of SFP instances. Investigating the use and fit of these approaches to solving~\eqref{P:dualSFP} is, however, beyond the scope of this paper.

\section{Applications}
	\label{S:Apps}

So far, we have focused on whether SFPs are tractable. In this section, we illustrate their expressiveness by using~\eqref{P:generalSFP} to cast the problems of nonlinear spectral estimation and robust functional data classification.

\subsection{Nonlinear line spectrum estimation}

The first example application of SFPs is in the context of continuous, possibly nonlinear, sparse dictionary recovery/denoising problems. Formally, let~$\by \in \setC^p$ collect samples~$y_i$, $i = 1,\dots,p$, of a signal. Our goal is to represent~$\by$ using as few atoms as possible from the nonlinear dictionary
\begin{equation}\label{E:dictionary}
	\calD = \left\{ \bF(\cdot, \bbeta): \setC \to \setC^p \mid \bbeta \in \Omega \right\}
		\text{.}
\end{equation}
Explicitly, we wish to find~$\{(x_k,\bbeta_k)\}$ such that
\begin{equation}\label{E:discreteSR}
	\byh = \sum_{k = 1}^K \bF(x_k, \bbeta_k)
\end{equation}
is close to~$\by$ for some small~$K$. Notice that, in contrast to classical dictionary recovery, the relation between the coefficients~$x_k$ and the signal~$\byh$ is not necessarily linear. Moreover, $\Omega$ is an uncountable set, so that we select from a continuum of atoms as opposed to the discrete, finite case.

To make the discussion concrete, consider the problem of estimating the parameters of a small number of saturated sinusoids from samples of their superposition. This problem is found in several signal processing applications, such as telecommunication and direction of arrival~(DOA) estimation, where nonlinear behaviors are common due to hardware limitations of the sources. Formally, we wish to estimate the frequencies, amplitudes, and phases of~$K$ sinusoids from the set of noisy samples
\begin{equation}\label{E:lseModel}
	y_i = \sum_{k = 1}^K \rho \left[ a_k \cos(2 \pi f_k t_i) \right]
		+ n_i
		\text{,} \quad \text{for }  i = 1,\dots,p
		\text{,}
\end{equation}
where~$f_k \in [0,1/2]$ is the frequency and~$a_k \in \setR$ is the amplitude/phase of the~$k$-th component; $t_i$ is the fixed, known sampling time of the $i$-th sample; $\{n_i\}$ are independent and identically distributed~(i.i.d.) zero-mean random variables with variance~$\E n_i^2 = \sigma_n^2$ representing the measurement noise; and~$\rho$ is a function that models the source nonlinearity with~$\rho(0) = 0$.

To pose this estimation problem as an SFP, we need an approximate continuous representation of the signal model in~\eqref{E:lseModel}. We say approximate because the nonlinearity~$\rho$ may prevent us from finding a measurable function~$X$ such that~$\int \rho \left[ X(\varphi) \cos(2 \pi \varphi t_i) \right] d\varphi = \rho\left[ x \cos(2 \pi f t_i) \right]$ for a fixed amplitude-frequency pair~$(x,f)$. Even if~$\rho$ allows it, an exact representation may involve Dirac deltas, which violates a hypothesis of Theorem~\ref{T:strongDuality} and prevents us from efficiently finding a solution of~\eqref{P:generalSFP}. The following proposition introduces a functional signal model that approximates~\eqref{E:lseModel} arbitrarily well using parameters in~$L_2$.

\begin{proposition}\label{T:lseDirect}

For fixed~$a,t \in \setR$, $f \in [0,1/2]$, define the hyperparameter~$B \in \setR_+$ and let
\begin{equation}\label{E:lseRepresentation}
	r(B) = B \int_{0}^{\frac{1}{2}}
		\rho \left[ X^\prime(\varphi) \cos(2 \pi \varphi t) \right] d\varphi
		\text{.}
\end{equation}
If~$X^\prime(\varphi) = a$ for~$\varphi \in [f-B^{-1},f+B^{-1}]$ and zero everywhere else, then~$r(B) \to \rho \left[ a \cos(2 \pi f t) \right]$ as~$B \to \infty$.

\end{proposition}

\begin{proof}

Note that~\eqref{E:lseRepresentation} is equivalent to
\begin{equation*}
	r(B) = \int_{0}^{\frac{1}{2}} B \cdot \Pi_{f,B^{-1}}(\varphi)
		\rho \left[ a \cos(2 \pi \varphi t) \right] d\varphi
\end{equation*}
with~$\Pi_{f,b}(\varphi) = \indicator\left( \varphi \in [f-b,f+b] \right)$. The result then follows from the fact that~$B \cdot \Pi_{f,B^{-1}}(\varphi)$ converges weakly to~$\delta(\varphi - f)$ as~$B \to \infty$, where~$\delta$ is the Dirac's delta~\cite{Rudin91f}.
\end{proof}

\begin{figure}[tb]
\centering
\includesvg{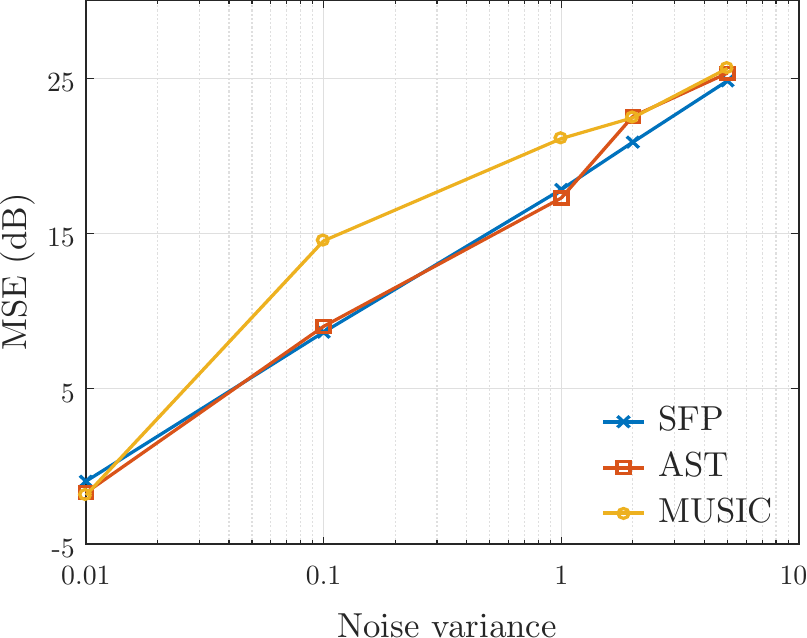}
\caption{Reconstruction MSE for line spectral estimation of linear sources.}
	\label{F:lse_mse}
\end{figure}

Proposition~\ref{T:lseDirect} allows us to cast nonlinear line spectrum estimation as the SFP
\begin{prob}\label{P:lseSFP}
	\minimize_{X \in L_2}& &&\norm{X}_{L_2} + \lambda \norm{X}_{L_0}
	\\
	\subjectto& &&\sum_{i = 1}^{p} (y_i - \hat{y}_i)^2 \leq \epsilon
	\\
	&&&\hat{y}_i = B \int_{0}^{\frac{1}{2}}
		\rho \left[ X(\varphi) \cos(2 \pi \varphi t_i) \right] d\varphi
		\text{,}
	\\
	&&&\qquad \text{for } i = 1,\dots,p
		\text{,}
\end{prob}
where~$B > 0$ is an approximation parameter and~$\epsilon > 0$ determines the solution fit. Problem~\eqref{P:lseSFP} explicitly seeks the sparsest function~$X$ that fits the observations given the model in~\eqref{E:lseModel}. The~$L_2$-norm regularization improves robustness to noise as well as the numerical properties of the dual by adding shrinkage. Note that since~$X \in L_2$, the solution~$X^\star$ of~\eqref{P:lseSFP} does not contain atoms and is instead a superposition of bump functions around the component frequencies~$f_k$~(see, e.g., Figure~\ref{F:nlse_x}). As Proposition~\ref{T:lseDirect} suggests, the height and width of each bump depends on the amplitude of the sinusoidal component and the choice of~$B$. Thus, the parameter~$a_k$ from~\eqref{E:lseModel} can be estimated using
\begin{equation}
	\hat{a}_k = B \int_{\calB_k} X^\star(\varphi) d\varphi
		\text{,}
\end{equation}
where~$X^\star$ is a solution of~\eqref{P:lseSFP} and~$\calB_k \subset [0,1/2]$ contains a single bump. The parameter~$f_k$ can then be estimated using the center frequency of the bump. Naturally, $B$ should be as large as possible so that~\eqref{E:lseRepresentation} is a good approximation of~\eqref{E:lseModel}, improving the parameter estimates. Choosing~$B$ too large, however, degrades the numerical properties of the dual problem, making it harder to solve in practice. Similar trade-offs are found several methods when tuning regularization parameters, for instance, elastic net~\cite{Hastie09e, Kuhn18a}.

Since~$\rho$ is an arbitrary function, a particular case of~\eqref{P:lseSFP} performs spectral estimation with linear sources, i.e., when~$\rho(z) = z$ in~\eqref{E:lseModel} and there is no saturation. The dual function is straightforward to evaluate in this case since the optimization problem~\eqref{E:gammao} from Proposition~\ref{T:dualXSolution} becomes a quadratic program that admits a closed-form solution. However, a myriad of classical methods such as MUSIC or atomic soft thresholding~(AST) have been proposed for the linear case. MUSIC performs line spectrum estimation using the eigendecomposition of the empirical autocorrelation matrix of the measurements~$y_i$~\cite{Stoica05s}. Nevertheless, it can only be used in single snapshot applications when the signal is sampled regularly---see~\cite{Stoica05s} for details---and requires that the number~$K$ of components be known \emph{a priori}. The AST approach, on the other hand, is based on an atomic norm relaxation of the sparse estimation problem and leverages duality and spectral properties of Toeplitz matrices to preclude discretization~\cite{Tang13c, Bhaskar13a}. Both methods first obtain the component frequencies and then determine amplitudes and phases using least squares. These different approaches are compared in Figures~\ref{F:lse_mse} and~\ref{F:lse_supp}.

\begin{figure}[tb]
\centering
\includesvg{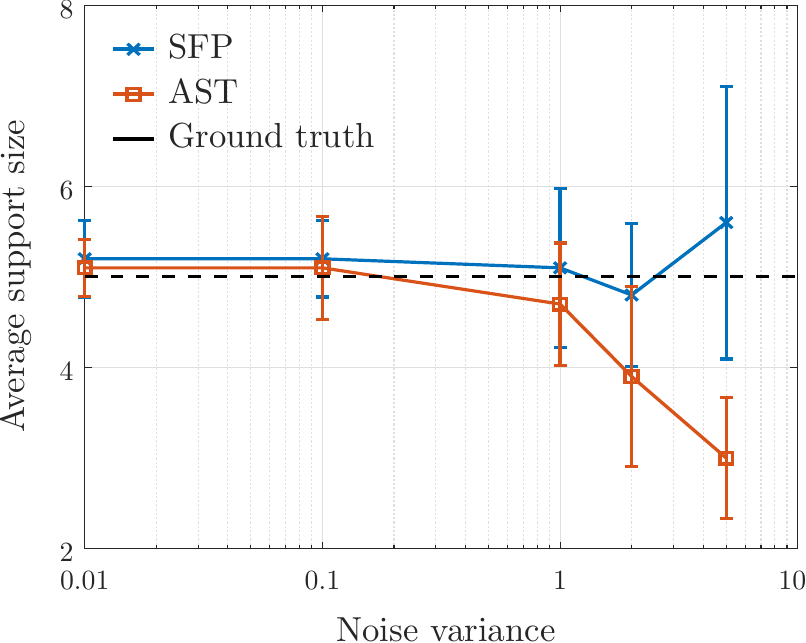}
\caption{Support size estimation for line spectral estimation of linear sources.}
	\label{F:lse_supp}
\end{figure}

These plots display the average performance over~$10$ realizations that used~$p = 61$ samples~($t_i = -30,\dots,30$) of the superposition of~$K = 5$ components whose the frequencies~$f_k$ were drawn uniformly at random with a minimum spacing of~$4/p$ and whose amplitudes~$a_k$ were taken randomly and independently from~$[0.5, 3]$. Problem~\eqref{P:lseSFP} was solved using the approximate supergradient method described in Appendix~\ref{A:solvingSFPs} with~$B = 1$, $\lambda = 5000$ for all noise levels expect~$\sigma_n^2 = 5$ which used~$\lambda = 6000$, and~$\epsilon = p \sigma_n^2$. For the AST method, we used the optimal regularization from~\cite{Bhaskar13a} which depends on~$\sigma_n^2$. In all cases, the reconstruction MSE is evaluated as
\begin{equation*}
	\MSE = \sum_{i = 1}^p (y_i - \hat{y}_i)^2
		\text{,}
\end{equation*}
where~$\hat{y}_i$ denotes the samples reconstructed based on the~$K$ components with largest magnitudes obtained by each algorithm. For AST, the support is obtained from the peaks of the trigonometric polynomial defined by the dual as in~\cite{Tang13c, Bhaskar13a} and for SFP, from the center of the bumps in the solution~$X^\star$ of~\eqref{P:lseSFP}~(as illustrated in Figure~\ref{F:nlse_x}).

In high SNR scenarios, all methods display similar performance. As the level of noise increases, however, the advantages of explicitly minimizing the~$L_0$-norm instead of its convex surrogate become clearer, especially with respect to support identification. Observe in Figure~\ref{F:lse_supp} that as~$\sigma_n^2$ increases the number of components obtained from AST decreases considerably, despite using the optimal regularization parameter. Finally, it is worth noting that although the performances are similar, AST involves solving a semidefinite program~(SDP), which becomes infeasible in practice as the number of samples~$p$ grows and has motivated the study of dimensionality reduction techniques and sampling patterns~\cite{Costa17s}. On the other hand, efficient solvers based on coordinate ascent can be leveraged to solve large-scale SFPs~\cite{Boyd04c, Bertsekas15c}.

\begin{figure}[tb]
\centering
\includesvg{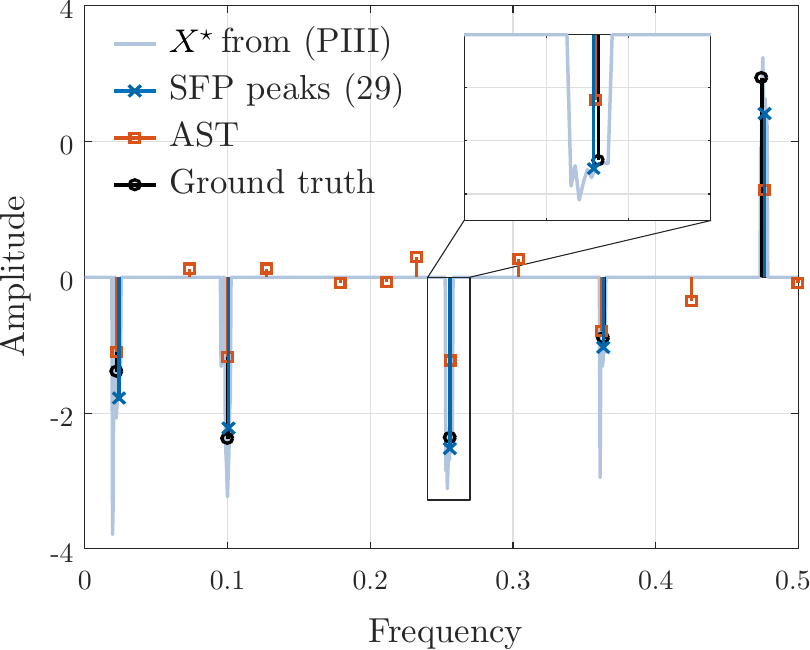}
\caption{Solutions obtained for line spectral estimation of saturated sources.}
	\label{F:nlse_x}
\end{figure}

Still, the signal reconstruction performance is similar across methods in the linear case~(Figure~\ref{F:lse_mse}). This is not surprising given the close relation between~$L_0$- and~$L_1$-norm minimization~(Theorem~\ref{T:L0L1}). In contrast, when the signals are distorted by a nonlinearity, the linear measurement model of AST and MUSIC tends to underestimate the amplitude of the components~(Figure~\ref{F:nlse_x}). Though greedy approaches to atomic norm minimization are able to deal with nonlinear dictionaries, optimally selecting single atoms from these infinite dimensional dictionary is challenging. Exhaustive, grid-based heuristics have been proposed for specific problems without guarantees~\cite{Rao15f}.

To illustrate this effect, consider the hard saturation
\begin{equation}\label{E:rho}
	\rho(x) = 
	\begin{cases}
		x \text{,} & \abs{x} \leq r
		\\
		r \cdot \sign(x) \text{,} & \text{otherwise}
	\end{cases}
		\text{,}
\end{equation}
where~$r > 0$ defines the saturation level. Though computing the dual function may seem challenging in this case due to the nonlinearity, it turns out to be tractable due to the scalar nature of the problem. Indeed, we obtain the dual minimizer from Proposition~\ref{T:dualXSolution} by evaluating
\begin{equation*}
	\gamma^o(\bmu,\varphi) = \min_{x \in \setR}\ x^2
		+ \bmu^T \rho \left[ x \bh(\varphi) \right]
		\text{,}
\end{equation*}
where~$[\bh]_i =  \cos(2 \pi \varphi t_i)$ for~$i = 1,\dots,p$ and the function~$\rho$ applies element-wise. Since we can determine \emph{a priori} which of the elements will saturate, solving this non-convex problem actually reduces to finding the minimum value of~$p$ quadratic problems. Namely, assume that~$\bh$ is sorted such that~$h_1 \leq \dots \leq h_p$ and let~$\bw_i(x) = [h_1 x\ \cdots\ h_i x\ r\ \cdots\ r]^T$, where~$r$ is the saturation level from~\eqref{E:rho}. For conciseness, we omit the dependence on~$\varphi$. Then, $\gamma^o(\bmu) = \min_{1 \leq i \leq p} \gamma^o_i(\bmu)$ for
\begin{equation*}
\begin{aligned}
	\gamma^o_i(\bmu) &= \min_{1/\abs{h_{i+1}} \leq \abs{x} \leq 1/\abs{h_{i}}}
		x^2 + \bmu^T \bw_i(x)
		\text{,} \quad i = 1,\dots,p-1
		\text{,}
	\\
	\gamma^o_p(\bmu) &= \min_{\abs{x} \leq 1/\abs{h_p}}
		x^2 + \bmu^T \bh x
		\text{.}
\end{aligned}
\end{equation*}

Figure~\ref{F:nlse_x} shows the solutions obtained using~\eqref{P:lseSFP} and AST for~$\rho$ as in~\eqref{E:rho} with~$r = 1$. We omit the results for MUSIC in this plot as its performance is similar to AST. Notice that since~\eqref{P:lseSFP} takes the the nonlinear nature of the signal into account it provides more precise parameter estimates. This is evident in Figure~\ref{F:nlse_mse}, which shows that~\eqref{P:lseSFP} leads to lower reconstruction errors, especially in higher SNRs. This is expected since neither AST nor MUSIC take the nonlinear effects into account. Yet, as the noise increases and begins to dominate over mismodeling, the performance of all methods becomes similar. This effect is more pronounced here than in the linear case because the saturation limits the energy of the signal leading to even lower effective SNRs. For instance, the average SNR for~$\sigma_n^2 = 2$ in Figure~\ref{F:lse_mse} is~$6.6$~dB, whereas in Figure~\ref{F:nlse_mse}, it is~$2.05$~dB.

In these experiments, the signal samples were constructed as in the linear case, but we used for~\eqref{P:lseSFP} $B = 200$, $\epsilon = p \sigma_n^2$, and~$\lambda = 100$ for all noise levels expect~$\sigma_n^2 \in \{2,5\}$ which used~$\lambda = 80$. For the AST method, we again used the optimal regularization parameter from~\cite{Bhaskar13a}. Better results could not be obtained by hand-tuning the regularization.

\begin{figure}[tb]
\centering
\includesvg{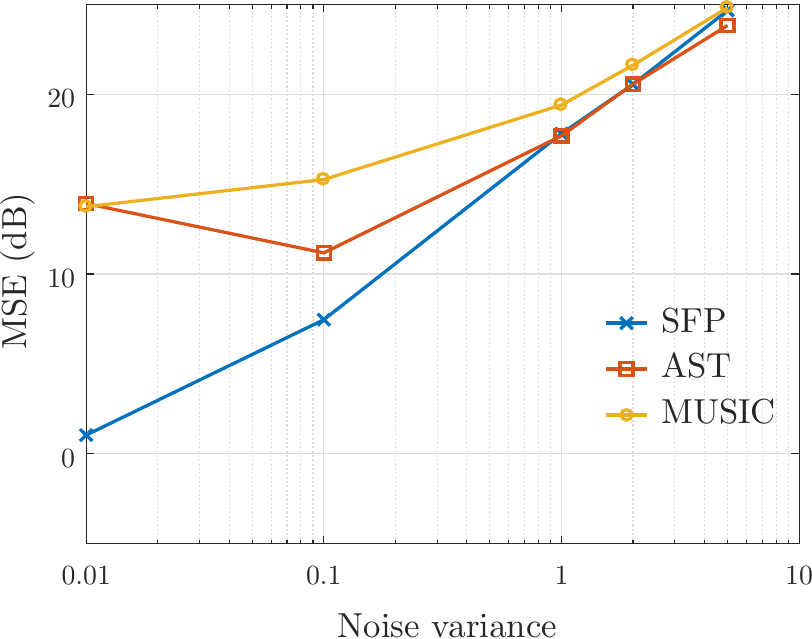}
\caption{Reconstruction MSE for line spectral estimation of saturated sources.}
	\label{F:nlse_mse}
\end{figure}

\subsection{Robust functional data analysis}

Functional data analysis extends classical statistical methods to data supported on continuous domains. Since it copes with non-uniformly sampled data and precludes registration, this tool set is especially appropriated for analyzing time series without assuming generative models, such as AR or ARMAX~\cite{Ramsay05f}. For concreteness, consider the functional extension of logistic regression: given a data pair~$(y_i, Z_i)$ with label~$y_i \in \{0,1\}$ and independent variable~$Z_i: [0,1] \to \setR$, the probability that~$y_i$ is positive is modeled as
\begin{equation}\label{E:FDA}
	\Pr\left[ y_i = 1 \right] =
		\frac{1}{1 + \exp\left( - \int_{0}^1 Z_i(\tau) W(\tau) d\tau + b \right)}
		\text{,}
\end{equation}
where~$W: [0,1] \to \setR$ is the functional classifier parameter and~$b$ is the intercept. Although the domain of~$Z_i$ and~$W$ can be an arbitrary compact set, we use the normalized~$[0,1]$ for simplicity. Typically, some smoothness prior is assumed for~$W$ so that the statistical problem is well-posed, e.g., by using splines or imposing that~$W$ has small RKHS norm~\cite{Ramsay05f}. Observe that if we replace~$\int_{0}^1 Z_i(\tau) W(\tau) d\tau$ by~$\bw^T \bz_i$, for~$\bw,\bz_i \in \setR^m$, we recover the classical, finite dimensional logistic model.

As is the case with traditional~(discrete) logistic regression, the classifier in~\eqref{E:FDA} is sensitive to outliers. In fact, it has been shown recently that any classifier trained by minimizing a convex loss, as is the case of logistic regression or support vector machines~(SVM), suffers from this issue~\cite{Chen13r, Feng14r}. Although sparsity has been used to mitigate this drawback using convex surrogates such as the $\ell_1$-norm~\cite{Plan13r, Tibshirani14r}, these methods remain susceptible to extreme data points caused by impulsive noise or other measurement errors~\cite{Feng14r}.

\begin{figure}[tb]
\centering
\includesvg{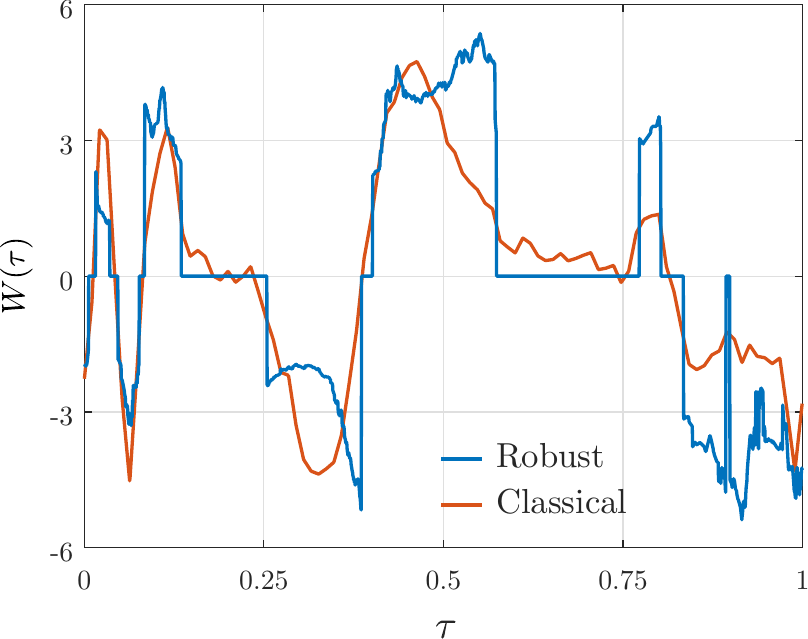}
\caption{Solution of functional logistic regression for ECG classification.}
	\label{F:weights}
\end{figure}

One approach to addressing this weakness is replacing the inner product in~\eqref{E:FDA} by a robust version that reduces the influence of these extreme samples. In~\cite{Chen13r, Feng14r}, this is done by computing inner products over a subset of the data. Here, however, since~\eqref{P:generalSFP} allows us to consider arbitrary nonlinearities in the data model, we can explicitly limit the influence of any sample by saturating the inner product in~\eqref{E:FDA}. Explicitly,
\begin{equation}\label{E:rFDA}
	\Pr\left[ y_i = 1 \right] =
		\frac{1}{1 + \exp\left( - \int_{0}^1 \rho\left[ Z_i(\tau) W(\tau) \right] d\tau + b \right)}
		\text{,}
\end{equation}
where~$\rho$ is the saturation from~\eqref{E:rho}. Notice that~\eqref{E:rFDA} controls the influence of any data point by using the threshold~$r$ from the saturation~\eqref{E:rho}. In fact, notice that due to the saturation, the value of the inner product in~\eqref{E:rFDA} lies in the range~$[-r,r]$. Using the negative log likelihood expression for logistic regression~\cite{Hastie09e}, we then formulate the following SFP for learning the robust classifier
\begin{prob}\label{P:rFDA}
	\minimize_{X \in L_2}& &&\norm{W}_{L_2} + b^2 + \lambda \norm{W}_{L_0}
	\\
	\subjectto& &&-\sum_{i = 1}^p \log\left[ 1 + \exp\left( (1-2y_i) \hat{y}_i \right) \right]
		\leq \epsilon
	\\
	&&&\hat{y}_i = \int_{\calT} \rho \left[ Z_i(\beta) W(\beta) \right] d\beta + b
		\text{,}
	\\
	&&&\text{for } i = 1,\dots,p
		\text{,}
\end{prob}
for some fit parameter~$\epsilon > 0$. Notice that~\eqref{P:rFDA} also allows us to fit sparse functional coefficient~$W$ by setting~$\lambda > 0$. Moreover, although it is written in terms of the logistic likelihood, other convex criteria such as the hinge loss could be used to obtain robust~SVMs.

\begin{figure}[tb]
\centering
\includesvg{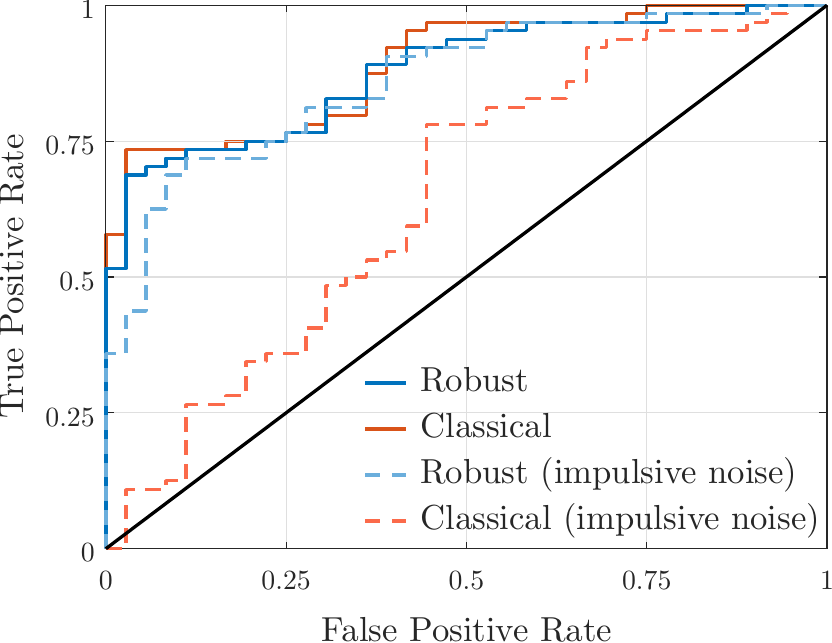}
\caption{Receiver operating characteristic~(ROC) curve for logistic classifiers in the presence of impulsive noise.}
	\label{F:roc}
\end{figure}

To illustrate the performance of the robust classifier~\eqref{E:rFDA}, we consider the problem of identifying whether an electrocardiogram~(ECG) signal comes from a healthy heart or one that suffered a myocardial infarction, i.e., a heart attack. The continuous time series~$Z_i$ are obtained by linearly interpolating a single heartbeat~(see examples in Figures~\ref{F:interp_normal} and~\ref{F:interp_mi}). Other techniques, such as sinc or spline interpolation are also commonly used in functional data analysis~\cite{Ramsay05f}. The labels~$y_i$ indicate whether a heart is healthy~($1$) or not~($0$). The samples used in the following experiments were taken from the ECG200 dataset~\cite{Bagnall17g, ecg200}, which draws from the MIT-BIH Supraventricular Arrhythmia Database~\cite{Goldberger00p}. To train the classical functional logistic classifier in~\eqref{E:FDA}, we solved~\eqref{P:rFDA} with~$\lambda = 0$ and~$r \to \infty$, i.e., no sparsity regularization and no saturation of the inner product. For the robust version, we used~$\lambda = 10$ and~$r = 4$. In both cases, the classifier was fitted with~$\epsilon = -46$ using the approximate supergradient method described in Appendix~\ref{A:solvingSFPs}.

Notice in Figure~\ref{F:weights} that the value of the coefficients of the classical and robust classifiers are similar, leading to comparable performance on both training and test sets~(approximately~$80\%$ accuracy). The receiver operating characteristic~(ROC) curve of both classifiers on the test set is displayed in solid lines in Figure~\ref{F:roc}. The robustness of these classifiers to outliers, on the other hand, is considerably different. To illustrate this behavior, corruption by impulsive noise was simulated by randomly adding~$\pm 20$ to a random subset of~$10\%$ of the samples from each heartbeat in the test set. The resulting ROC curves are shown in dashed lines. Although the performance of the linear logistic classifier has now degraded~(the test accuracy dropped to~$66\%$), the ROC of the robust version remains unaltered due to the nonlinearity~$\rho$ in~\eqref{E:rFDA} limiting the effect of the corruption~(test accuracy of~$76\%$).

Additionally, the sparsity of the robust classifier parameters improves interpretability by focusing on the portions of the signal that differentiate between normal and abnormal heartbeats~(Figures~\ref{F:interp_normal} and~\ref{F:interp_mi}). For instance, healthy heart signals tend to have negative values for~$\tau \in [0.25,0.4]$ and positive values for~$\tau \in [0.4,0.6]$, whereas hearts that suffered myocardial infarctions do not. On the other hand, there is no discriminant information for~$\tau \in [0.6,0.75]$ and, perhaps less intuitively, between~$0.15$ and~$0.25$.

\begin{figure}[tb]
\centering
\includesvg{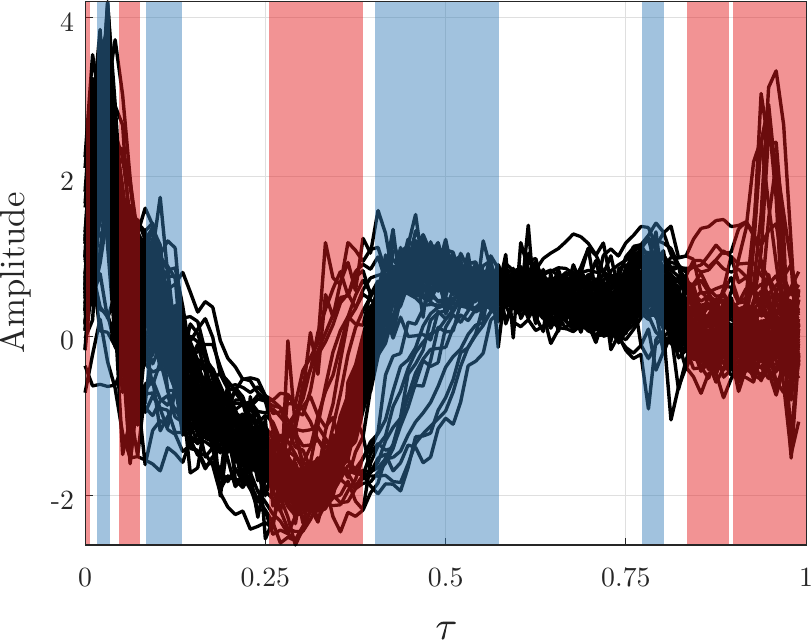}
\caption{ECG of \emph{healthy heart} and sparse functional coefficients~(positive coefficients: blue, negative coefficients: red).}
	\label{F:interp_normal}
\end{figure}

\section{Conclusion}

We proposed to tackle nonlinear, continuous problems involving sparsity penalties directly by solving sparse functional optimization problems. To do so, we showed that a large class of these mathematical programs have no duality gap and can therefore be solved by means of their dual problems. Duality simultaneously bypasses the infinite dimensionality and non-convexity hurdles of the original problem and enables the use of efficient algorithms to solve these non-convex functional programs. Signal processing applications~(nonlinear line spectral estimation and robust functional logistic regression) were used to illustrate the expressiveness of this technique, that we foresee can be used to solve a wide variety of problems in different domains. Future work includes investigating second-order stochastic optimization algorithms to improve the convergence rate of Algorithm~\ref{L:dualAscent} and obtaining identifiability/recovery results for problems such as line spectral estimation. We also believe these strong duality results apply to problems beyond sparsity.

\begin{figure}[tb]
\centering
\includesvg{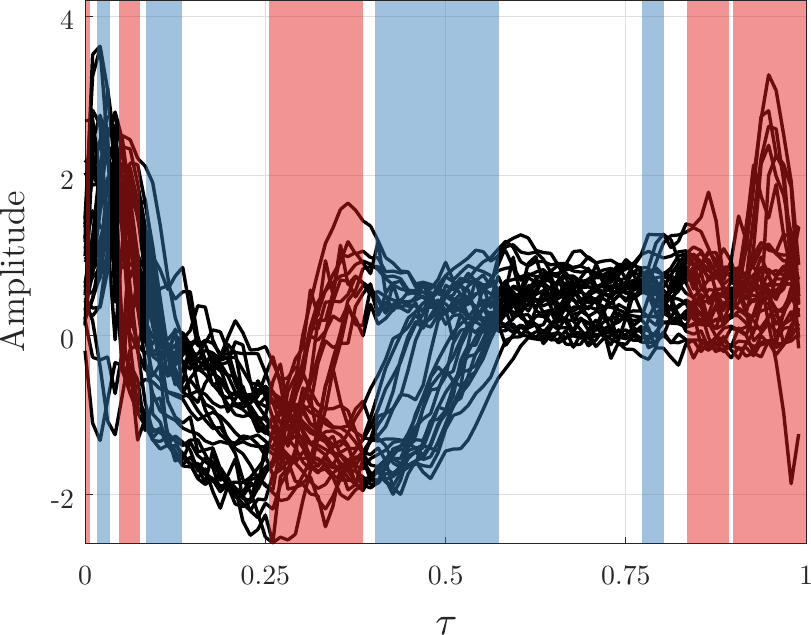}
\caption{ECG of heart with \emph{myocardial infraction} and sparse functional weights~(positive coefficients: blue, negative coefficients: red).}
	\label{F:interp_mi}
\end{figure}

\bibliographystyle{IEEEbib}
\bibliography{IEEEabrv,gsp,sp,math,stat}

\appendices

\section{Proof of Lemma~\ref{T:convexC}}
\label{A:convexC}

Let~$(c,\bu,\bk_R,\bk_I),(c^\prime,\bu^\prime,\bk_R^\prime,\bk_I^\prime)$ be arbitrary points in~$\calC$ achieved for~$(X,\bz),(X^\prime,\bz^\prime) \in \calX \times \setC^p$. In other words, it holds that~$X(\bbeta),X^\prime(\bbeta) \in \calP$~a.e., $f_0(X) \leq c$, $f_0(X^\prime) \leq c^\prime$, $g_i(\bz) \leq [\bu]_i$, $g_i(\bz^\prime) \leq [\bu^\prime]_i$,
\begin{align*}
	\int_{\Omega} \bF \left[ X(\bbeta), \bbeta \right] d\bbeta - \bz
		&= \bk_R + j\bk_I \triangleq \bk
		\text{,}
	\\
	\int_{\Omega} \bF \left[ X^\prime(\bbeta), \bbeta \right] d\bbeta - \bz^\prime
		&= \bk_R^\prime + j\bk_I^\prime \triangleq \bk^\prime
		\text{,}
\end{align*}
where we have defined the shorthands~$\bk$ and~$\bk^\prime$ for conciseness. To show that~$\calC$ is convex, it suffices to show that~$\theta (c,\bu,\bk_R,\bk_I) + (1-\theta) (c^\prime,\bu^\prime,\bk_R^\prime,\bk_I^\prime) \in \calC$ for any~$\theta \in [0,1]$. Equivalently, we must obtain~$(X_\theta,\bz_\theta) \in \calX \times \setC^p$ such that~$X_\theta(\bbeta) \in \calP$ a.e.,
\begin{subequations}\label{E:convexC}
\begin{align}
	f_0(X_\theta) &\leq \theta c + (1-\theta) c^\prime
		\label{E:convexC1}
		\text{,}
	\\
	g_i(\bz_\theta) &\leq [\theta \bu + (1-\theta) \bu^\prime]_i
		\label{E:convexC2}
		\text{,}
	\\
	\int_{\Omega} \bF \left[ X_\theta(\bbeta), \bbeta \right] d\bbeta - \bz_\theta &=
		\theta \bk + (1-\theta) \bk^\prime	
		\label{E:convexC3}
		\text{,}
\end{align}
\end{subequations}
for any~$0 \leq \theta \leq 1$. To do so, we will rely on the following classical theorem about the range of non-atomic vector measures:

\begin{theorem}[Lyapunov's convexity theorem~\cite{Diestel77v}]
\label{T:lyapunov}

Let~$\fkv: \calB \to \setC^n$ be a vector measure over the measurable space~$(\Omega,\calB)$. If~$\fkv$ is non-atomic, then its range is convex, i.e., the set~$\{ \fkv(\calA) : \calA \in \calB \}$ is a convex set.

\end{theorem}

To see how Theorem~\ref{T:lyapunov} allows us to construct the desired~$X_\theta$, start by defining a~$2(p+1) \times 1$ vector measure~$\fkp$ over~$(\Omega,\calB)$ such that for every set~$\calZ \in \calB$ we have
\begin{equation}\label{E:vectorMeas}
	\renewcommand*{\arraystretch}{1.5}
	\fkp(\calZ) =
	\begin{bmatrix}
		\int_{\calZ} \bF \left[ X(\bbeta), \bbeta \right] d\bbeta
		\\
		\int_{\calZ} \bF \left[ X^\prime(\bbeta), \bbeta \right] d\bbeta
		\\
		\int_{\calZ} \left[ F_0 \left( X(\bbeta),\bbeta \right)
			+ \lambda \indicator\left( X(\bbeta) \neq 0 \right) \right] d\bbeta
		\\
		\int_{\calZ} \left[ F_0 \left( X^\prime(\bbeta),\bbeta \right)
			+ \lambda \indicator\left( X^\prime(\bbeta) \neq 0 \right) \right] d\bbeta
	\end{bmatrix}
		\text{.}
\end{equation}
Notice that~$\fkp$ is a proper vector measure, so that~$\fkp(\emptyset) = \bzero$. Also, observe that evaluating~$\fkp$ on the whole space~$\Omega$ yields
\begin{equation}\label{E:vectorMeasOmega}
	\renewcommand*{\arraystretch}{1.5}
	\fkp(\Omega) =
	\begin{bmatrix}
		\int_{\Omega} \bF \left[ X(\bbeta), \bbeta \right] d\bbeta
		\\
		\int_{\Omega} \bF \left[ X^\prime(\bbeta), \bbeta \right] d\bbeta
		\\
		f_0(X)
		\\
		f_0(X^\prime)
	\end{bmatrix}
	\Rightarrow
	\fkp(\Omega)
	=
	\begin{bmatrix}
		\bk + \bz
		\\
		\bk^\prime + \bz^\prime
		\\
		f_0(X)
		\\
		f_0(X^\prime)
	\end{bmatrix}
		\text{.}
\end{equation}
Finally, observe that since~$F_0$ and~$\bF$ do not contain Dirac deltas, they induce non-atomic measures. Consequently, $\fkp$ is non-atomic.

To proceed, use Theorem~\ref{T:lyapunov} to find a set~$\calT_\theta \in \calB$ such that
\begin{equation}\label{E:calT}
	\fkp(\calT_\theta)
		= \theta \fkp(\Omega) + (1-\theta) \fkp(\emptyset)
		= \theta \fkp(\Omega)
\end{equation}
for~$\theta \in [0,1]$. Since~$\calB$ is a $\sigma$-algebra, it holds that~$\Omega \setminus \calT_\theta \in \calB$ and by the additivity of measures we get
\begin{equation}\label{E:calTc}
	\fkp(\Omega \setminus \calT_\theta) =
		(1 - \theta) \fkp(\Omega)
		\text{.}
\end{equation}
From~\eqref{E:calT} and~\eqref{E:calTc}, construct~$X_\theta$ as
\begin{equation}\label{E:xTheta}
	X_\theta(\bbeta) =
	\begin{cases}
		X(\bbeta) \text{,}
			&\text{for } \bbeta \in \calT_\theta
		\\
		X^\prime(\bbeta) \text{,}
			&\text{for } \bbeta \in \Omega \setminus \calT_\theta
	\end{cases}
\end{equation}
and let~$\bz_\theta = \theta \bz + (1-\theta) \bz^\prime$. We claim that this pair satisfies~\eqref{E:convexC}. It is straightforward from the fact that~$\calX$ is decomposable and that~$\calP$ is a pointwise constraint, that~$X_\theta \in \calX$ and $X_\theta(\bbeta) \in \calP$~a.e.

Let us start by showing that~$X_\theta$ satisfies~\eqref{E:convexC1}. Evaluating~$f_0$ at~$X_\theta$ yields
\begin{align*}
	f_0(X_\theta) &= \int_{\Omega} \left[ F_0 \left( X_\theta(\bbeta),\bbeta \right)
		+ \lambda \indicator\left( X_\theta(\bbeta) \neq 0 \right) \right] d\bbeta
	\\
	{}&= \int_{\calT_\theta} \left[ F_0 \left( X(\bbeta),\bbeta \right)
			+ \lambda \indicator\left( X(\bbeta) \neq 0 \right) \right] d\bbeta
	\\
	{}&+ \int_{\Omega \setminus \calT_\theta} \left[ F_0 \left( X^\prime(\bbeta),\bbeta \right)
			+ \lambda \indicator\left( X^\prime(\bbeta) \neq 0 \right) \right] d\bbeta
		\text{.}
\end{align*}
From~\eqref{E:vectorMeas}, we can write these terms using the last two rows of the vector measure~$\fkp$ as
\begin{equation}\label{E:machinery}
	f_0(X_\theta) = \left[ \mathfrak{p}(\calT_\theta) \right]_{2p+1} +
		\left[ \mathfrak{p}(\Omega \setminus \calT_\theta) \right]_{2p+2}
		\text{.}
\end{equation}
Then, using~\eqref{E:calT} and~\eqref{E:calTc} we obtain that
\begin{align*}
	f_0(X_\theta) &= \left[ \theta \mathfrak{p}(\Omega) \right]_{2p+1} +
		\left[ (1 - \theta) \mathfrak{p}(\Omega) \right]_{2p+2}
	\\
	{}&= \theta f_0(X) + (1 - \theta) f_0(X^\prime) \leq \theta c + (1 - \theta) c^\prime
		\text{.}
\end{align*}

To proceed, notice that since the~$g_i$ are convex functions, \eqref{E:convexC2} obtains immediately. Explicitly,
\begin{align*}
	g_i(\bz_\theta) &= g_i(\theta \bz + (1-\theta) \bz^\prime)
		\leq \theta g_i(\bz) + (1-\theta) g_i(\bz^\prime)
	\\
	{}&\leq [\theta \bu + (1-\theta) \bu^\prime]_i
		\text{.}
\end{align*}
Finally, we can use the same machinery as in~\eqref{E:machinery} to obtain~\eqref{E:convexC3}. Indeed,
\begin{align*}
	\int_{\Omega} \bF \left[ X_\theta(\bbeta), \bbeta \right] d\bbeta
		&= \int_{\calT_\theta} \bF \left[ X(\bbeta), \bbeta \right] d\bbeta
	\\
	{}&+ \int_{\Omega \setminus \calT_\theta}
		\bF \left[ X^\prime(\bbeta), \bbeta \right] d\bbeta
	\\
	{}&= \left[ \fkp(\calT_\theta) \right]_{\calS_1}
		+ \left[ \fkp(\Omega \setminus \calT_\theta) \right]_{\calS_2}
	\\
	{}&= \left[ \theta \fkp(\Omega) \right]_{\calS_1}
		+ \left[ (1-\theta) \fkp(\Omega) \right]_{\calS_2}
	\\
	{}&= \theta \bk + (1 - \theta) \bk^\prime + \bz_\theta
		\text{,}
\end{align*}
where~$\calS_1 = \{1,\dots,p\}$ and~$\calS_2 = \{p+1,\dots,2p\}$ select rows~$1$ through~$p$ and~$p+1$ through~$2p$, respectively, of the vector measure~$\fkp$. 

To conclude, since there exists a pair~$(X_\theta,\bz_\theta) \in \calX \times \setC^p$ with~$X_\theta(\bbeta) \in \calP$ a.e.\ and such that~\eqref{E:convexC} holds for any~$\theta \in [0,1]$ and~$(c,\bu,\bk_R,\bk_I),(c^\prime,\bu^\prime,\bk_R^\prime,\bk_I^\prime) \in \calC$, the set~$\calC$ is convex. Moreover, the strictly feasible pair~$(X^\prime,\bz^\prime)$ from the hypotheses implies that~$\calC$ cannot be empty.

\section{A step-by-step guide to solving SFPs}
\label{A:solvingSFPs}

Start with a problem of the form~\eqref{P:generalSFP}. Initialize~$\bmu_0$ and~$\nu_{i,0} > 0$; compute~$d_{\bz,0} = \min_{\bz} \sum_i \nu_{i,0} g_i(\bz) - \Re[ \bmu_{0}^H \bz ]$ and let~$\bz_{0}$ be its minimizer; evaluate
\begin{equation}\label{E:gamma_o_0}
	\gamma^o_{0}(\bbeta) = \min_{x \in \calP} F_0(x,\bbeta)
		+ \Re \left[ \bmu_{0}^H \bF(x, \bbeta) \right]
\end{equation}
and let~$\bar{X}_{0}(\bbeta)$ be its minimizer; define the initial solution support to be~$\calS_{0} = \{\bbeta \in \Omega : \gamma^o_{0}(\bbeta) < \gamma^{(0)}(\bmu_{0},\bbeta) - \lambda\}$ for~$\gamma^{(0)}$ defined as in Proposition~\ref{T:dualXSolution}; obtain the primal solution
\begin{equation*}
	X_{0}(\bbeta) = \begin{cases}
		\bar{X}_{0}(\bbeta) \text{,} & \bbeta \in \calS_{0}
		\\
		0 \text{,} & \text{otherwise}
	\end{cases}
\end{equation*}
and evaluate the initial dual objective using
\begin{align*}
	d_{0} &= d_{\bz,0} + I \left[ \left( \lambda + \gamma^o_{0}(\bbeta) \right) \times
		\indicator \left( \bbeta \in \calS_0 \right) \vphantom{\sum}\right]
	\\
	{}&+ I \left[ \gamma^{(0)}(\bmu_{0},\bbeta) \times
		\indicator \left( \bbeta \in \Omega \setminus \calS_0 \right) \vphantom{\sum}\right]
		\text{,}
\end{align*}
where~$I$ denotes a numerical integration method. Then, proceed using one of the following solvers.

\vspace*{\baselineskip}
\noindent\textbf{Approximate supergradient ascent}.
Consider a numerical integration procedure represented by~$I(\cdot)$ such that
\begin{equation}\label{E:integralApproxError}
	\abs{I( f ) - \int_{\Omega} f(\bbeta) d\bbeta} \leq \delta
		\text{,}
\end{equation}
for~$\delta > 0$ and assume that~$I$ applies element-wise to vectors. Let~$X^\star_{0} = X_0$ and for~$t = 1,\dots,T$:

\begin{enumerate}[i)]

\item compute the supergradients
\begin{align*}
	\bp_{\bmu,t-1} &=
		I \left[ \bF \left( X_{t-1}(\bbeta), \bbeta \right) \vphantom{\sum}\right]
		- \bz_{t-1}
	\\
	\bp_{\nu_i,t-1} &= g_i\left[ \bz_{t-1} \right]
		\text{,}
\end{align*}

\item update the dual variables
\begin{align*}
	\bmu_t &= \bmu_{t-1} + \eta_t \bp_{\bmu,t-1}
	\\
	\nu_{i,t} &= \left[ \nu_{i,t-1} + \eta_t g_i\left( \bz_{t-1} \right) \right]_+
		\text{,}
\end{align*}

\item evaluate~$d_{\bz,t} = \min_{\bz} \sum_i \nu_{i,t} g_i(\bz) - \Re[ \bmu_{t}^H \bz ]$ and let~$\bz_{t}$ be its minimizer,

\item evaluate
\begin{equation*}
	\gamma^o_{t}(\bbeta) = \min_{x \in \calP} F_0(x,\bbeta)
		+ \Re \left[ \bmu_{t}^H \bF(x, \bbeta) \right]
		\text{,}
\end{equation*}
and let~$\bar{X}_{t}(\bbeta)$ be its minimizers,

\item evaluate the dual function
\begin{align*}
	d_{t} &= d_{\bz,t} + I \left[ \left( \lambda + \gamma^o_{t}(\bbeta) \right) \times
		\indicator \left( \bbeta \in \calS_t \right) \vphantom{\sum}\right]
	\\
	{}&+ I \left[ \gamma^{(0)}(\bmu_{t},\bbeta) \times
		\indicator \left( \bbeta \in \Omega \setminus \calS_t \right) \vphantom{\sum}\right]
		\text{,}
\end{align*}
for~$\calS_{t} = \{\bbeta \in \Omega : \gamma^o_{t}(\bbeta) < \gamma^{(0)}(\bmu_{t},\bbeta) - \lambda\}$, and

\item if~$d_{t} > d_{t-1} + 2 \delta$, obtain the primal solution
\begin{equation*}
	X_{t}(\bbeta) = \begin{cases}
		\bar{X}_{t}(\bbeta) \text{,} & \bbeta \in \calS_{t}
		\\
		0 \text{,} & \text{otherwise}
	\end{cases}
\end{equation*}
and let~$X^\star_t = X_t$. Otherwise, $X^\star_t = X^\star_{t-1}$.

\end{enumerate}

The solution of~\eqref{P:generalSFP} is given by~$X^\star_T$.

\vspace*{\baselineskip}
\noindent\textbf{Stochastic supergradient ascent}.
Choose the mini-batch size~$N \geq 1$ and initialize the solution set~$\calX_0 = \emptyset$. For~$t = 1,\dots,T$:

\begin{enumerate}[i)]

\item draw~$\{\bbeta_j\}$, $j = 1,\dots,N$, uniformly at random from~$\Omega$ and compute the stochastic supergradients
\begin{align*}
	\bp_{\bmu,t-1} &=
		\frac{1}{N} \sum_{j = 1}^N \bF \left[ X_{t-1}(\bbeta_j), \bbeta \right] d\bbeta
		- \bz_{t-1}
	\\
	\bp_{\nu_i,t-1} &= g_i\left[ \bz_{t-1} \right]
		\text{,}
\end{align*}

\item update the dual variables
\begin{align*}
	\bmu_t &= \bmu_{t-1} + \eta_t \bp_{\bmu,t-1}
	\\
	\nu_{i,t} &= \left[ \nu_{i,t-1} + \eta_t g_i\left( \bz_{t-1} \right) \right]_+
		\text{,}
\end{align*}

\item evaluate~$d_{\bz,t} = \min_{\bz} \sum_i \nu_{i,t} g_i(\bz) - \Re[ \bmu_{t}^H \bz ]$ and let~$\bz_{t}$ be its minimizer,

\item evaluate
\begin{equation*}
	\gamma^o_{t}(\bbeta) = \min_{x \in \calP} F_0(x,\bbeta)
		+ \Re \left[ \bmu_{t}^H \bF(x, \bbeta) \right]
\end{equation*}
and let~$\bar{X}_{t}(\bbeta)$ be its minimizer,

\item evaluate the dual function
\begin{equation*}
	d_{t} = d_{\bz,t} + \int_{\calS_{t}} \left[ \lambda + \gamma^o_{t}(\bbeta) \right] d\bbeta
		+ \int_{\Omega \setminus \calS_{t}} \gamma^{(0)}(\bmu_{t},\bbeta) d\bbeta
		\text{,}
\end{equation*}
for~$\calS_{t} = \{\bbeta \in \Omega : \gamma^o_{t}(\bbeta) < \gamma^{(0)}(\bmu_{t},\bbeta) - \lambda\}$, and

\item if~$d_{t} > d_{t-1}$, obtain the primal solution
\begin{equation*}
	X_{t-1}(\bbeta) = \begin{cases}
		\bar{X}_{t-1}(\bbeta) \text{,} & \bbeta \in \calS_{t-1}
		\\
		0 \text{,} & \text{otherwise}
	\end{cases}
\end{equation*}
and let~$\calX_t = \calX_{t-1} \cup X_t$. Otherwise, $\calX_t = \calX_{t-1}$.

\end{enumerate}

The final solution is obtained by averaging the elements of~$\calX_T$, i.e.,
\begin{equation*}
	X^\star(\bbeta) = \frac{1}{\abs{\calX_T}} \sum_{X \in \calX_T} X(\bbeta)
		\text{.}
\end{equation*}

\section{Proof of Proposition~\ref{T:numericalIntegration}}
\label{A:numericalIntegration}

We actually prove the following quantitative version of Proposition~\ref{T:numericalIntegration}:

\begin{proposition}

Under the conditions of Proposition~\ref{T:numericalIntegration}, $\abs{P^\star - P_\delta^\star} \leq c \delta + o(\delta^2)$ for
\begin{equation}\label{E:errorConstant}
	c = \frac{\bar{F}_0
			+ \lambda \fkm(\Omega)}{\alpha \epsilon}
	\max\left(
		\abs{\sum_i g_i(-\alpha \ones)},
		\abs{\sum_i g_i(\alpha \ones)}
	\right)
		\text{,}
\end{equation}
where~$P^\star$ is the optimal value of~\eqref{P:generalSFP} and~$P_\delta^\star$ is the value of the solution obtained by Algorithm~\ref{L:dualAscent} when evaluating the integral in the supergradient~\eqref{E:gradMu} with approximation error~$0 < \delta \ll 1$~[as in~\eqref{E:integralApproxError}].

\end{proposition}

\begin{proof}

Start by noticing that evaluating the integral in~\eqref{E:gradMu} numerically introduces an error term in the supergradient. Explicitly, \eqref{E:gradMu} becomes
\begin{equation}\label{E:gtilde}
	\tilde{g}_{\bmu}(\bmu^\prime,\nu_i^\prime) =
		\int_{\Omega} \bF \left[ X_d(\bmu^\prime, \bbeta), \bbeta \right] d\bbeta
		- \bz_d(\bmu^\prime,\nu_i^\prime)
		+ \bdelta
		\text{,}
\end{equation}
where~$\bdelta$ is an error vector whose magnitude is bounded by $\delta$, i.e., $\abs{[\bdelta]_i} < \delta$. Then, observe that~\eqref{E:gtilde} is the supergradient of the dual function of a perturbed version of~\eqref{P:generalSFP}, namely
\begin{prob}\label{P:numericalIntegral}
	\minimize&
		&&\int_{\Omega} F_0 \left[ X(\bbeta),\bbeta \right] d\bbeta
			+ \lambda \norm{X}_{L_0}
	\\
	\subjectto& &&g_i(\bz) \leq 0
	\\
	&&&\bz = \int_{\Omega} \bF \left[ X(\bbeta), \bbeta \right] d\bbeta + \bdelta
	\\
	&&&X \in \calX
		\text{.}
\end{prob}
Hence, the value~$P_\delta^\star$ of the solution obtained by the using approximate supergradient in Algorithm~\ref{L:dualAscent} is the optimal value of~\eqref{P:numericalIntegral}. We can therefore use perturbation theory to relate the values of~$P_\delta^\star$ and~$P^\star$.

Formally, using the fact that the perturbation function of~\eqref{P:generalSFP} is differentiable around zero~[hypothesis~(i)], we obtain the Taylor expansion~$P_\delta^\star = P^\star - {\bmu^\star}^T \bdelta + o(\norm{\bdelta}_2^2)$, where~$o(t)$ is a term such that~$o(t)/t \to 0$ as~$t \to 0$~\cite{Boyd04c}. Hence, using the triangle inequality and the upper bound on the elements of~$\abs{\bdelta}$, we can write
\begin{equation}\label{E:perturbationBound}
	\abs{P^\star - P_\delta^\star} = \abs{{\bmu^\star}^T \bdelta + o(\norm{\bdelta}_2^2)}
		\leq \abs{{\bmu^\star}^T \ones} \delta + o(\delta^2)
		\text{.}
\end{equation}
It suffices now to bound~$\abs{{\bmu^\star}^T \ones}$, which we do in two steps.

First, we obtain an upper bound on~${\bmu^\star}^T \ones$ by recalling from~\eqref{E:preDualFunction1} that the dual function~$d$ is the value of a minimization problem. Thus, taking the suboptimal~$X \equiv 0$ and~$\bz = \alpha \ones$, $\alpha > 0$, under hypothesis~(ii) yields
\begin{equation*}
	d(\bmu^\star,\nu_i^\star) \leq \sum_i \nu_i^\star g_i(\alpha \ones) - \alpha {\bmu^\star}^T \ones
		\text{.}
\end{equation*}
From Theorem~\ref{T:strongDuality}, $d(\bmu^\star,\nu_i^\star) = P^\star \geq 0$, which gives
\begin{equation}\label{E:upperBound}
	{\bmu^\star}^T \ones \leq \frac{\sum_i \nu_i^\star g_i(\alpha \ones)}{\alpha}
		\text{.}
\end{equation}
Proceeding in a similar manner, we derive a lower bound by taking~$X \equiv 0$ and~$\bz = -\alpha \ones$ in~\eqref{E:preDualFunction1}, leading to
\begin{equation}\label{E:lowerBound}
	{\bmu^\star}^T \ones \geq -\frac{\sum_i \nu_i^\star g_i(-\alpha \ones)}{\alpha}
		\text{.}
\end{equation}
Using the Cauchy-Schwartz inequality, the bounds in~\eqref{E:upperBound} and~\eqref{E:lowerBound} yield
\begin{equation}\label{E:jointBound}
	\abs{{\bmu^\star}^T \ones} \leq \frac{\norm{\bm{\nu}^\star}_1}{\alpha}
	\max\left(
		\abs{\sum_i g_i(-\alpha \ones)},
		\abs{\sum_i g_i(\alpha \ones)}
	\right)
		\text{,}
\end{equation}
where~$\bm{\nu}^\star = \left[ \nu_i^\star \right]$ is a vector that collects the optimal dual variables~$\nu_i^\star$. Note that since~$\nu_i^\star \geq 0$, we have that~$\abs{\sum_i \nu_i^\star} = \norm{\bm{\nu}^\star}_1$. All that remains to evaluate~\eqref{E:jointBound} is to bound~$\norm{\bm{\nu}^\star}_1$ using a classical result from optimization theory.

Explicitly, consider the strictly feasible pair~$(X^\dagger,\bz^\dagger)$ from hypothesis~(iv) and recall that~$g_i(\bz^\dagger) \leq -\epsilon$ for some~$\epsilon > 0$. Plugging these suboptimal values in~\eqref{E:preDualFunction1} yields
\begin{equation}\label{E:preNuBound}
	d(\bmu^\star,\nu_i^\star) \leq \int_{\Omega} F_0 \left[ X^\dagger(\bbeta),\bbeta \right] d\bbeta
		+ \lambda \norm{X^\dagger}_{L_0}
	- \sum_i \nu_i^\star \epsilon
		\text{.}
\end{equation}
Recall that~$\nu_i^\star \geq 0$, $\epsilon > 0$, and~$d(\bmu^\star,\nu_i^\star) \geq 0$~(from~Theorem~\ref{T:strongDuality}). Thus, using the fact that~$\norm{X^\dagger}_{L_0} \leq \fkm(\Omega)$, we readily obtain from~\eqref{E:preNuBound} that
\begin{equation}\label{E:nuBound}
	\norm{\bm{\nu}^\star}_1 \leq
	\frac{
		\int_{\Omega} F_0 \left[ X^\dagger(\bbeta),\bbeta \right] d\bbeta
			+ \lambda \fkm(\Omega)
	}{
		\epsilon
	}
		\text{.}
\end{equation}
Combining~\eqref{E:jointBound} and~\eqref{E:nuBound} in~\eqref{E:perturbationBound} we obtain that~$\abs{P^\star - P_\delta^\star} \leq c \delta + o(\delta^2)$ for~$c$ as in~\eqref{E:errorConstant}. Furthermore, hypotheses~(iii) and~(iv), together with~$\fkm(\Omega) < \infty$~(since~$\Omega$ is compact), imply that~$c < \infty$, so that indeed~$\abs{P^\star - P_\delta^\star} \leq \calO(\delta)$.
\end{proof}

% ========== BIOGRAPHIES ==========
\begin{IEEEbiography}[{\includegraphics[width=1in,height=1.25in]{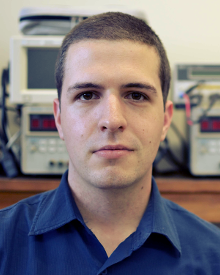}}]{Luiz F. O. Chamon (S'12)}
received the B.Sc. and M.Sc. degree in electrical engineering from the University of S\~{a}o Paulo, S\~{a}o Paulo, Brazil, in 2011 and 2015. In 2009, he was an undergraduate exchange student at the Masters in Acoustics of the \'{E}cole Centrale de Lyon, Lyon, France. He is currently working toward the Ph.D. degree in electrical and systems engineering at the University of Pennsylvania (Penn), Philadelphia. In 2009, he was an Assistant Instructor and Consultant on nondestructive testing at INSACAST Formation Continue. From 2010 to 2014, he worked as a Signal Processing and Statistical Consultant on a project with EMBRAER. His research interest include signal processing, optimization, statistics, and control.
\end{IEEEbiography}

\begin{IEEEbiography}[{\includegraphics[width=1in,height=1.25in]{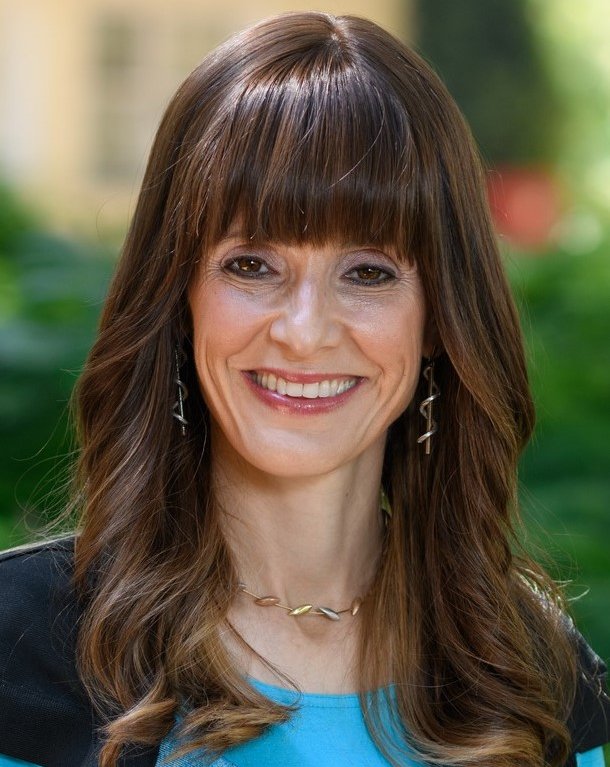}}]{Yonina C. Eldar (S’98–M’02–SM’07-F’12)}
received the B.Sc.\ degree in Physics in~1995 and the B.Sc.\ degree in Electrical Engineering in~1996 both from Tel-Aviv University~(TAU), Tel-Aviv, Israel, and the Ph.D.\ degree in Electrical Engineering and Computer Science in~2002 from the Massachusetts Institute of Technology~(MIT), Cambridge. She is currently a Professor in the Department of Mathematics and Computer Science, Weizmann Institute of Science, Rehovot, Israel. She was previously a Professor in the Department of Electrical Engineering at the Technion, where she held the Edwards Chair in Engineering. She is also a Visiting Professor at MIT, a Visiting Scientist at the Broad Institute, and an Adjunct Professor at Duke University and was a Visiting Professor at Stanford. She is a member of the Israel Academy of Sciences and Humanities~(elected~2017), an IEEE Fellow and a EURASIP Fellow. Her research interests are in the broad areas of statistical signal processing, sampling theory and compressed sensing, learning and optimization methods, and their applications to biology and optics.

Dr.\ Eldar has received numerous awards for excellence in research and teaching, including the IEEE Signal Processing Society Technical Achievement Award~(2013), the IEEE/AESS Fred Nathanson Memorial Radar Award~(2014), and the IEEE Kiyo Tomiyasu Award~(2016). She was a Horev Fellow of the Leaders in Science and Technology program at the Technion and an Alon Fellow. She received the Michael Bruno Memorial Award from the Rothschild Foundation, the Weizmann Prize for Exact Sciences, the Wolf Foundation Krill Prize for Excellence in Scientific Research, the Henry Taub Prize for Excellence in Research~(twice), the Hershel Rich Innovation Award~(three times), the Award for Women with Distinguished Contributions, the Andre and Bella Meyer Lectureship, the Career Development Chair at the Technion, the Muriel \& David Jacknow Award for Excellence in Teaching, and the Technion’s Award for Excellence in Teaching~(twice). She received several best paper awards and best demo awards together with her research students and colleagues including the SIAM outstanding Paper Prize and the IET Circuits, Devices and Systems Premium Award, and was selected as one of the 50 most influential women in Israel.

She was a member of the Young Israel Academy of Science and Humanities and the Israel Committee for Higher Education. She is the Editor in Chief of Foundations and Trends in Signal Processing, a member of the IEEE Sensor Array and Multichannel Technical Committee and serves on several other IEEE committees. In the past, she was a Signal Processing Society Distinguished Lecturer, member of the IEEE Signal Processing Theory and Methods and Bio Imaging Signal Processing technical committees, and served as an associate editor for the IEEE Transactions on Signal Processing, the EURASIP Journal of Signal Processing, the SIAM Journal on Matrix Analysis and Applications, and the SIAM Journal on Imaging Sciences. She was Co-Chair and Technical Co-Chair of several international conferences and workshops.
\end{IEEEbiography}

\begin{IEEEbiography}[{\includegraphics[width=1in,height=1.25in]{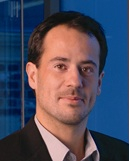}}]{Alejandro Ribeiro}
received the B.Sc. degree in electrical engineering from the Universidad de la Republica Oriental del Uruguay, Montevideo, in 1998 and the M.Sc. and Ph.D. degree in electrical engineering from the Department of Electrical and Computer Engineering, the University of Minnesota, Minneapolis in 2005 and 2007.

From 1998 to 2003, he was a member of the technical staff at Bellsouth Montevideo. After his M.Sc. and Ph.D studies, in 2008 he joined the University of Pennsylvania (Penn), Philadelphia, where he is currently the Rosenbluth Associate Professor at the Department of Electrical and Systems Engineering. His research interests are in the applications of statistical signal processing to the study of networks and networked phenomena. His focus is on structured representations of networked data structures, graph signal processing, network optimization, robot teams, and networked control.

Dr. Ribeiro received the 2014 O. Hugo Schuck best paper award, and paper awards at CDC 2017, 2016 SSP Workshop, 2016 SAM Workshop, 2015 Asilomar SSC Conference, ACC 2013, ICASSP 2006, and ICASSP 2005. His teaching has been recognized with the 2017 Lindback award for distinguished teaching and the 2012 S. Reid Warren, Jr. Award presented by Penn's undergraduate student body for outstanding teaching. Dr. Ribeiro is a Fulbright scholar class of 2003 and a Penn Fellow class of 2015.
\end{IEEEbiography}

\end{document}